\documentclass{article}

\usepackage[final,nonatbib]{neurips_2024}

\usepackage[utf8]{inputenc} 
\usepackage[T1]{fontenc}    
\usepackage{xcolor}         
\definecolor{linkdarkblue}{rgb}{0, 0.08, 0.45}    
\usepackage[colorlinks, anchorcolor=white, linkcolor=linkdarkblue, urlcolor=linkdarkblue, citecolor=linkdarkblue]{hyperref}
\usepackage{url}            
\usepackage{booktabs}       
\usepackage{array}          
\usepackage{colortbl}       
\usepackage{makecell}       
\usepackage{tabularx}       
\usepackage{bigstrut}       
\usepackage{multirow}       
\usepackage{colortbl}       
\usepackage{nicefrac}       
\usepackage{microtype}      
\usepackage[pdftex]{graphicx} 
\usepackage{amsmath}        
\usepackage{amsfonts}       
\usepackage{amssymb}        
\usepackage{amsthm}         
\usepackage{bm}             
\usepackage[capitalise,noabbrev]{cleveref} 
\usepackage{subcaption}     
\usepackage{booktabs}       
\usepackage{floatrow}
\newfloatcommand{capbtabbox}{table}[][\FBwidth]
\usepackage{wrapfig}        
\usepackage{pifont}         
\usepackage[toc,page,header]{appendix} 
\usepackage{algorithm}      
\usepackage{mathtools}      
\usepackage{enumitem}       
\usepackage{pifont}         
\usepackage{algorithm}      
\usepackage{algorithmicx}   
\usepackage{algpseudocode}  
\usepackage[numbers,comma,compress]{natbib} 

\definecolor{darkgreen}{rgb}{0.0, 0.5, 0.0}
\definecolor{lightblue}{rgb}{0.867, 0.922, 0.969}

\numberwithin{equation}{section}

\newtheorem{theorem}{Theorem}[section]

\newtheorem{lemma}[theorem]{Lemma}

\newcommand{\ie}{\text{i.e., }}         
\newcommand{\eg}{\text{e.g., }}         
\newcommand{\etc}{\text{etc}}           
\newcommand{\wrt}{\text{w.r.t. }}       
\newcommand{\cf}{\text{cf. }}           
\newcommand{\aka}{\text{a.k.a. }}       

\title{PSL: Rethinking and Improving Softmax Loss from Pairwise Perspective for Recommendation}

\author{
    Weiqin Yang~\footnotemark[2]\hspace{0.5em}\footnotemark[3] \\
    Zhejiang University \\
    \texttt{tinysnow@zju.edu.cn} \\
    \And
    Jiawei Chen~\footnotemark[1]\hspace{0.5em}\footnotemark[2]\hspace{0.5em}\footnotemark[3]\hspace{0.5em}\footnotemark[4] \\
    Zhejiang University \\
    \texttt{sleepyhunt@zju.edu.cn} \\
    \And
    Xin Xin \\
    Shandong University \\
    \texttt{xinxin@sdu.edu.cn} \\
    \AND
    Sheng Zhou \\
    Zhejiang University \\
    \texttt{zhousheng\_zju@zju.edu.cn} \\
    \And
    Binbin Hu \\
    Ant Group \\
    \texttt{bin.hbb@antfin.com} \\
    \And
    Yan Feng~\footnotemark[2]\hspace{0.5em}\footnotemark[3] \\
    Zhejiang University \\
    \texttt{fengyan@zju.edu.cn} \\
    \AND
    Chun Chen~\footnotemark[2]\hspace{0.5em}\footnotemark[3] \\
    Zhejiang University \\
    \texttt{chenc@zju.edu.cn} \\
    \And
    Can Wang~\footnotemark[2]\hspace{0.5em}\footnotemark[4] \\
    Zhejiang University \\
    \texttt{wcan@zju.edu.cn} \\
}

\begin{document}

\maketitle

\renewcommand{\thefootnote}{\fnsymbol{footnote}}

\footnotetext[1]{Corresponding author.}
\footnotetext[2]{State Key Laboratory of Blockchain and Data Security, Zhejiang University.}
\footnotetext[3]{College of Computer Science and Technology, Zhejiang University.}
\footnotetext[4]{Hangzhou High-Tech Zone (Binjiang) Institute of Blockchain and Data Security.}

\renewcommand{\thefootnote}{\arabic{footnote}}
\setcounter{footnote}{0}

\begin{abstract}
Softmax Loss (SL) is widely applied in recommender systems (RS) and has demonstrated effectiveness. This work analyzes SL from a pairwise perspective, revealing two significant limitations: 1) the relationship between SL and conventional ranking metrics like DCG is not sufficiently tight; 2) SL is highly sensitive to false negative instances. Our analysis indicates that these limitations are primarily due to the use of the exponential function. To address these issues, this work extends SL to a new family of loss functions, termed Pairwise Softmax Loss (PSL), which replaces the exponential function in SL with other appropriate activation functions. While the revision is minimal, we highlight three merits of PSL: 1) it serves as a tighter surrogate for DCG with suitable activation functions; 2) it better balances data contributions; and 3) it acts as a specific BPR loss enhanced by Distributionally Robust Optimization (DRO). We further validate the effectiveness and robustness of PSL through empirical experiments. The code is available at \url{https://github.com/Tiny-Snow/IR-Benchmark}.
\end{abstract}

\section{Introduction} \label{sec:introduction}

Nowadays, recommender systems (RS) have permeated various personalized services \citep{ko2022survey,zhang2019deep,huang2023aligning,huang2024large}. What sets recommendation apart from other machine learning tasks is its distinctive emphasis on ranking \citep{liu2009learning}. Specifically, RS aims to retrieve positive items in higher ranking positions (\ie giving larger prediction scores) over others and adopts specific ranking metrics (\eg DCG \citep{jarvelin2017ir} and MRR \citep{lu2023optimizing}) to evaluate its performance.

The emphasis on ranking inspires a surge of research on loss functions in RS. Initial studies treated recommendation primarily as a classification problem, utilizing pointwise loss functions (\eg BCE \citep{he2017ncf}, MSE \citep{he2017nfm}) to optimize models. Recognizing the inherent ranking nature of RS, pairwise loss functions (\eg BPR \citep{rendle2009bpr}) were introduced to learn a partial ordering among items. More recently, Softmax Loss (SL) \citep{wu2024effectiveness} has integrated contrastive learning paradigms \citep{liu2021self, wu2024understanding}, augmenting positive items as compared with negative ones, achieving state-of-the-art (SOTA) performance.

While SL has proven effective, it still suffers from \textbf{two limitations}: 1) SL can be used to approximate ranking metrics, \eg DCG and MRR \citep{wu2024effectiveness, bruch2019analysis}, but their relationships are not sufficiently tight. Specifically, SL uses the exponential function $\exp(\cdot)$ as the \emph{surrogate activation} to approximate the Heaviside step function in DCG, resulting in a notable gap, especially when the surrogate activation takes larger values. 2) SL is sensitive to noise (\eg false negatives \citep{wu2023bsl}). Gradient analysis reveals that SL assigns higher weights to negative instances with large prediction scores, while the weights are rather skewed and governed by the exponential function. This characteristic renders the model highly sensitive to false negative noise. Specifically, false negative instances are common in RS, as a user's lack of interaction with an item might stem from unawareness rather than disinterest \citep{chen2023bias,chen2021autodebias,wang2024llm4dsr}. These instances would receive disproportionate emphasis, potentially dominating the training direction, leading to performance degradation and training instability.

To address these challenges, we propose a new family of loss functions, termed \textbf{Pairwise Softmax Loss (PSL)}. PSL first reformulates SL in a pairwise manner, where the loss is applied to the score gap between positive-negative pairs. Such pairwise perspective is more fundamental to recommendation as the ranking metrics are also pairwise dependent. Recognizing that the primary weakness of SL lies in its use of the exponential function, PSL replaces this with other surrogate activations. While this extension is straightforward, it brings significant theoretical merits:

\begin{itemize}[topsep=0pt,leftmargin=10pt]
    \setlength{\itemsep}{0pt}
    \item \textbf{Tighter surrogate for ranking metrics.} We establish theoretical connections between PSL and conventional ranking metrics, \eg DCG. By choosing appropriate surrogate activations, such as ReLU or Tanh, we demonstrate that PSL achieves a tighter DCG surrogate loss than SL.
    \item \textbf{Control over the weight distribution.} PSL provides flexibility in choosing surrogate activations that control the weight distribution of training instances. By substituting the exponential function with an appropriate surrogate activation, \eg ReLU or Tanh, PSL can mitigate the excessive impact of false negatives, thus enhancing robustness to noise.
    \item \textbf{Theoretical connections with BPR loss.} Our analyses reveal that optimizing PSL is equivalent to performing Distributionally Robust Optimization (DRO) \citep{shapiro2017distributionally} over the conventional pairwise loss BPR \citep{rendle2009bpr}. DRO is a theoretically sound framework where the optimization is not only on a fixed empirical distribution but also across a set of distributions with adversarial perturbations. This DRO characteristic endows PSL with stronger generalization and robustness against out-of-distribution (OOD), especially given that such distribution shifts are common in RS, \eg shifts in user preference and item popularity \citep{chen2023bias,zhang2023invariant,zhao2022popularity}.
\end{itemize}

Our analyses underscore the theoretical effectiveness and robustness of PSL. To empirically validate these advantages, we implement PSL with typical surrogate activations (Tanh, Atan, ReLU) and conduct extensive experiments on four real-world datasets across three experimental settings: 1) IID setting \citep{he2020lightgcn} where training and test distributions are identically distributed \citep{mohri2018foundations}; 2) OOD setting \citep{wang2024distributionally} with distribution shifts in item popularity; 3) Noise setting \citep{wu2023bsl} with a certain ratio of false negatives. Experimental results demonstrate the superiority of PSL over existing losses in terms of recommendation accuracy, OOD robustness, and noise resistance.


\section{Preliminaries} \label{sec:review-softmax}

\textbf{Task formulation.}
We will conduct our discussion in the scope of collaborative filtering (CF) \citep{su2009survey}, a widely-used recommendation scenario. Given the user set $\mathcal{U}$ and item set $\mathcal{I}$, CF dataset $\mathcal{D} \subset \mathcal{U} \times \mathcal{I}$ is a collection of observed interactions, where each instance $(u, i) \in \mathcal{D}$ means that user $u$ has interacted with item $i$ (\eg clicks, reviews, \etc). For each user $u$, we denote $\mathcal{P}_u = \{ i \in \mathcal{I} : (u, i) \in \mathcal{D}\}$ as the set of positive items of $u$, while $\mathcal{I} \setminus \mathcal{P}_u$ represents the negative items. 

The goal of recommendation is to learn a recommendation model, or essentially a scoring function $f(u, i) : \mathcal{U} \times \mathcal{I} \to \mathbb{R}$ that quantifies the preference of user $u$ on item $i$ accurately. Modern RS often adopts an embedding-based paradigm \citep{koren2009matrix}. Specifically, the model maps user $u$ and item $i$ into $d$-dim embeddings $\mathbf{u}, \mathbf{v} \in \mathbb{R}^{d}$, and predicts their preference score $f(u, i)$ based on embedding similarity. The cosine similarity is commonly utilized in RS and has demonstrated particular effectiveness \citep{chen2023adap}. Here we set $f(u, i) = \frac{\mathbf{u} \cdot \mathbf{v}}{\Vert \mathbf{u}\Vert \Vert \mathbf{v} \Vert} \cdot \frac{1}{2}$, where the scaling factor $\frac{1}{2}$ is introduced for faciliating analyses and can be absorbed into the temperature hyperparameter ($\tau$). The scores $f(u, i)$ are subsequently utilized to rank items for generating recommendations.

\textbf{Ranking metrics.}
The Discounted Cumulative Gain (DCG) \citep{jarvelin2017ir} is a prominent ranking metric for evaluating the recommendation quality. Formally, for each user $u$, DCG is calculated as follows: 
\begin{equation} \label{eq:DCG}
    \mathrm{DCG}(u) = \sum_{i \in \mathcal{P}_u} \frac{1} { \log_2(1 + \pi_u(i))}
\end{equation}
where $\pi_u(i)$ is the ranking position of item $i$ in the ranking list sorted by the scores $f(u, i)$. DCG quantifies the cumulative gain of positive items, discounted by their ranking positions. Similarly, the Mean Reciprocal Rank (MRR) \citep{lu2023optimizing,argyriou2020microsoft} is another popular ranking metric using the reciprocal of the ranking position as the gain, \ie $\mathrm{MRR}(u) = \sum_{i \in \mathcal{P}_u} 1 / \pi_u(i)$. Additionally, other metrics such as Recall \citep{fayyaz2020recommendation}, Precision \citep{fayyaz2020recommendation}, and AUC \citep{silveira2019good} are also utilized in RS  \citep{fayyaz2020recommendation}. Compared to these metrics, DCG and MRR focus more on the top-ranked recommendations, thus attracting increasing attention in RS \citep{wu2024effectiveness,rashed2021guided}. In this work, we aim to explore the surrogate loss for DCG and MRR.

\textbf{Recommendation losses.}
To train recommendation models effectively, a series of recommendation losses has been developed. Recent work on loss functions can mainly be classified into three types:
\begin{itemize}[topsep=0pt,leftmargin=10pt]
    \setlength{\itemsep}{0pt}
    \item \textbf{Pointwise loss} (\eg BCE \citep{he2017ncf}, MSE \citep{he2017nfm}, \etc.) formulates recommendation as a specific classification or regression task, and the loss is applied to each positive and negative instance separately. Specifically, for each user $u$, the pointwise loss is defined as
    \begin{equation} \label{eq:pointwise}
        \mathcal{L}_{\textnormal{pointwise}}(u) = -\sum_{i \in \mathcal{P}_u} \log(\varphi^+(f(u, i))) - \sum_{j \in \mathcal{I} \setminus \mathcal{P}_u} \log(\varphi^-(f(u, j)))
    \end{equation}
    where $\varphi^+(\cdot)$ and $\varphi^-(\cdot)$ are the activation functions adapted for different loss choices.
    \item \textbf{Pairwise loss} (\eg BPR \citep{rendle2009bpr}, \etc.) optimizes partial ordering among items, which is applied to the score gap between negative-positive pairs. BPR \citep{rendle2009bpr} is a representative pairwise loss, which is defined as
    \begin{equation} \label{eq:bpr}
        \mathcal{L}_{\textnormal{BPR}}(u) = \sum_{i \in \mathcal{P}_u} \sum_{j \in \mathcal{I} \setminus \mathcal{P}_u} \log \sigma(f(u, j) - f(u, i))
    \end{equation}
    where $\sigma$ denotes the activation function that approximates the Heaviside step function. The basic intuition behind BPR loss is to let the positive instances have higher scores than negative instances. In practice, there are various choices of the activation function. For instance, \citet{rendle2009bpr} originally uses the sigmoid function, and the resultant BPR loss can approximate AUC metric.
    \item \textbf{Softmax Loss} (\ie SL \citep{wu2024effectiveness}) normalizes the predicted scores into a multinomial distribution \citep{casella2024statistical} and optimizes the probability of positive instances over negative ones \citep{cao2007l2r}, which is defined as
    \begin{equation} \label{eq:softmax-original}
        \mathcal{L}_{\textnormal{SL}}(u) = 
        -\sum_{i \in \mathcal{P}_u} \log\left(\frac{\exp(f(u, i) / \tau)}{\sum_{j \in \mathcal{I}} \exp(f(u, j) / \tau)}\right)
    \end{equation}
    where $\tau$ is the temperature hyperparameter. SL can also be understood as a specific contrastive loss, which draws positive instances $(u,i)$ closer and pushes negative instances $(u,j)$ away \citep{wu2024understanding}.
\end{itemize}

\section{Analyses on Softmax Loss from Pairwise Perspective} \label{sec:analysis-softmax}
In this section, we aim to first represent the Softmax Loss (SL) in a pairwise form, followed by an analysis of its relationship with the DCG metric, where two limitations of SL are exposed.

\textbf{Pairwise form of SL.}
To facilitate the analysis of SL and to build its relationship with the DCG metric, we rewrite SL (\cf \cref{eq:softmax-original}) in the following pairwise form:
\begin{equation} \label{eq:softmax-pairwise}
    \mathcal{L}_{\textnormal{SL}}(u) = 
    \sum_{i \in \mathcal{P}_u} \log \left(\sum_{j \in \mathcal{I}} \exp(d_{uij} / \tau)\right)
    , \quad
    \text{where }
    d_{uij} = f(u, j) - f(u, i)
\end{equation}
\cref{eq:softmax-pairwise} indicates that SL is penalized based on the score gap between negative-positive pairs, \ie $d_{uij} = f(u,j) - f(u,i)$. This concise expression is fundamental for ranking, as it optimizes the relative order of instances rather than their absolute values. 

\textbf{Connections between SL and DCG.}
We now analyze the connections between SL and the DCG metric (\cf \cref{eq:softmax-pairwise,eq:DCG}), which could enhance our understanding of the advantages and disadvantages of SL. Our analysis follows previous work \citep{wu2024effectiveness,bruch2019analysis}, which begins by relaxing the negative logarithm of DCG with
\begin{equation} \label{eq:softmax-dcg-surrogate-bound}
    -\log \mathrm{DCG}(u) + \log |\mathcal{P}_u|
    \leq -\log \left(\dfrac{1}{|\mathcal{P}_u|} \sum_{i \in \mathcal{P}_u} \dfrac{1}{\pi_u(i)} \right)
    \leq \dfrac{1}{|\mathcal{P}_u|} \sum_{i \in \mathcal{P}_u} \log \pi_u(i)
\end{equation}
where the first inequality holds due to $\log_2(1 + \pi_u(i)) \leq \pi_u(i)$, and the second inequality holds due to Jensen's inequality \citep{jensen1906fonctions}. Note that the ranking position $\pi_u(i)$ of item $i$ can be expressed as
\begin{equation} \label{eq:softmax-dcg-pi}
    \pi_u(i) 
    = \sum_{j \in \mathcal{I}} \mathbb{I}(f(u, j) \geq  f(u, i)) 
    = \sum_{j \in \mathcal{I}} \mathbb{\delta}(d_{uij}) 
\end{equation}
where $\delta(\cdot)$ denotes the Heaviside step function, with $\delta(x)=1$ for $x \geq 0$ and $\delta(x)=0$ for $x < 0$. Since $\delta(d_{uij}) \leq \exp(d_{uij} / \tau)$ holds for all $\tau > 0$, we deduce that SL is a smooth upper bound of \cref{eq:softmax-dcg-surrogate-bound}, and thus serves as a reasonable surrogate loss for DCG and MRR metrics\footnote{Note that the middle term in \cref{eq:softmax-dcg-surrogate-bound}, \ie $-\log \left(\frac{1}{|\mathcal{P}_u|} \sum_{i \in \mathcal{P}_u} {1}/{\pi_u(i)} \right)$, is exactly $-\log \mathrm{MRR}(u)$. Therefore, SL serves as an upper bound of the negative logarithm of DCG and MRR, and minimizing SL leads to the improvement of these ranking metrics.}.

However, our analysis also reveals \textbf{two limitations of SL}:

\begin{itemize}[topsep=0pt,leftmargin=10pt]
    \setlength{\itemsep}{0pt}
    \item \textbf{Limitation 1: SL is not tight enough as a DCG surrogate loss.}
    There remains a significant gap between the Heaviside step function $\delta(\cdot)$ and the exponential function $\exp(\cdot)$, especially when $d_{uij}$ reaches a relatively large value, where $\exp(\cdot)$ becomes substantially larger than $\delta(\cdot)$. This gap is further exacerbated by the temperature $\tau$. Practically, we find that the optimal $\tau$ is usually chosen to be less than 0.2 (\cf \cref{sec:appendix-experiments-best-hyperparameters}). Given the explosive nature of $\exp(\cdot)$, the gap becomes extremely large, potentially leading to suboptimal performance of SL in optimizing DCG.
    \item \textbf{Limitation 2: SL is highly sensitive to noise (\eg false negative instances).}
    False negative instances \citep{wu2023bsl} are common in the typical RS. This is often due to the exposure bias \citep{chen2023bias}, where a user's lack of interaction with an item might stem from unawareness rather than disinterest. Unfortunately, SL is highly sensitive to these false negative instances. On one hand, these instances $(u, j)$, which may exhibit patterns similar to true positive ones, are difficult for the model to differentiate and often receive larger predicted scores, thus bringing potentially larger $d_{uij}$ for positive items $i$. As analyzed in Limitation 1, these instances can significantly enlarge the gap between SL and DCG due to the exponential function, causing the optimization to deviate from the DCG metric.
\end{itemize}

\textbf{Gradient analysis of SL.}
Another perspective to support the view of Limitation 2 comes from the gradient analysis. Specifically, the gradient of SL \wrt $d_{uij}$ is
\begin{equation} \label{eq:softmax-gradient}
    \frac{\partial \mathcal{L}_{\textnormal{SL}}(u)}{\partial d_{uij}} 
    = \frac{\exp(d_{uij} / \tau) / \tau}{|\mathcal{I}|\mathbb{E}_{j' \sim \mathcal{I}} [\exp(d_{uij'} / \tau)]}
    \textcolor{darkgreen}{\quad \propto \exp(d_{uij} / \tau) / \tau}
\end{equation}
As can be seen, SL implicitly assigns a weight to the gradient of each negative-positive pair, where the weight is proportional to $\exp(d_{uij} / \tau)$. This suggests that instances with larger $d_{uij}$ will receive larger weights. While this property may be desirable for hard mining \citep{wu2024effectiveness}, which can accelerate convergence, it also means that false negative instances, which typically have larger $d_{uij}$, will obtain disproportionately large weights, as shown in the weight distribution of SL in \cref{fig:psl-activation-d}. Therefore, the optimization of SL can be easily dominated by false negative instances, leading to performance drops and training instability.

\textbf{Discussions on DRO robustness and noise sensitivity.}
Recent work \citep{wu2023bsl} claims that SL exhibits robustness to noisy data through Distributionally Robust Optimization (DRO) \citep{shapiro2017distributionally}. However, we argue that this is not the case. DRO indeed can enhance model robustness to distribution shifts, but it also increases the risk of noise sensitivity, as demonstrated by many studies on DRO \citep{zhai2021doro,nietert2024outlier}. Intuitively, DRO emphasizes hard instances with larger losses, making noisy data contribute more rather than less to the optimization. This is also demonstrated from the experiments with false negative instances (\cf Figure 8 in \citep{wu2023bsl}), where the improvements of SL over other baselines in Noise setting do not increase significantly but sometimes decay.


\section{Methodology} \label{sec:psl}

\subsection{Pairwise Softmax Loss} \label{sec:psl-definition-analysis}

Recognizing the limitations of SL, particularly its reliance on the unsatisfactory exponential function, we propose to extend SL with a more general family of losses, termed \textbf{Pairwise Softmax Loss (PSL)}. In PSL, the exponential function $\exp(\cdot)$ is replaced by other \emph{surrogate activations} $\sigma(\cdot)$ approximating the Heaviside step function $\delta(\cdot)$. For each user $u$, the PSL is defined as
\begin{equation} \label{eq:psl}
    \mathcal{L}_{\textnormal{PSL}}(u) = 
    \sum_{i \in \mathcal{P}_u} \log\left(\sum_{j \in \mathcal{I}} \sigma(d_{uij})^{{1}/{\tau}}
    \right)
\end{equation}
One might wonder why we apply the temperature outside the activation function (\ie extending $\exp(d_{uij})^{1/\tau}$ to $\sigma(d_{uij})^{1/\tau}$ )\footnote{Note that the equation $\exp(d_{uij}/\tau)=\exp(d_{uij})^{1/\tau}$ holds.} rather than within it (\ie extending $\exp(d_{uij}/\tau)$ to $\sigma(d_{uij}/\tau)$). This subtlety will be elucidated later as we demonstrate that the form in \cref{eq:psl} offers superior properties over the alternative.

Our PSL provides a flexible framework for selecting better activation functions, allowing the loss to exhibit improved properties compared to SL. We advocate for three activations, including \textbf{PSL-tanh}: $\sigma_{\textnormal{tanh}} = \tanh(d_{uij}) + 1$, \textbf{PSL-atan}: $\sigma_{\textnormal{atan}} = \arctan(d_{uij}) + 1$, and \textbf{PSL-relu}: $\sigma_{\textnormal{relu}} = \mathrm{ReLU}(d_{uij} + 1)$. In the following, we will discuss the advantages of PSL and provide evidence for the selection of these surrogate activations.

\textbf{Advantage 1: PSL is a better surrogate for ranking metrics.} To highlight the advantages of replacing $\exp(\cdot)$ with alternative surrogate activations, we present the following lemma:

\begin{lemma} \label{lemma:psl-dcg-surrogate}
If the condition
\begin{equation} \label{eq:psl-activation-condition}
    \delta(d_{uij}) \leq \sigma(d_{uij}) \leq \exp(d_{uij})
\end{equation}
is satisfied for any $d_{uij} \in [-1, 1]$, then PSL serves as a tighter DCG surrogate loss compared to SL.
\end{lemma}

The proof is presented in Appendix \ref{sec:appendix-proof-psl-dcg-surrogate}. This lemma reveals that PSL could be a tighter surrogate loss for DCG compared to SL. Additionally, it provides guidance on the selection of a proper surrogate activation --- we may choose the activation that lies between $\exp(\cdot)$ and $\delta(\cdot)$. As demonstrated in \cref{fig:psl-activation}, our chosen surrogate activations $\sigma_{\textnormal{tanh}}$, $\sigma_{\textnormal{atan}}$, and $\sigma_{\textnormal{relu}}$ adhere to this principle. 

\begin{figure}[t]
    \centering
    \begin{subfigure}{0.49\textwidth}
        \centering
        \includegraphics[width=\textwidth]{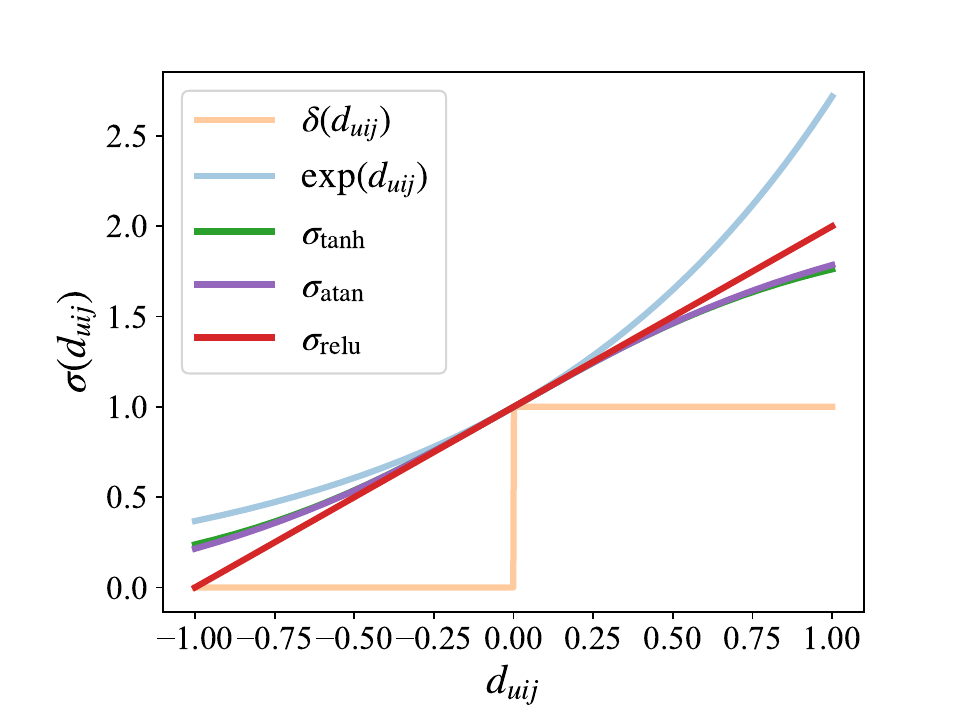}
        \caption{Surrogate activations.}
        \label{fig:psl-activation}
    \end{subfigure}
    \begin{subfigure}{0.49\textwidth}
        \centering
        \includegraphics[width=\textwidth]{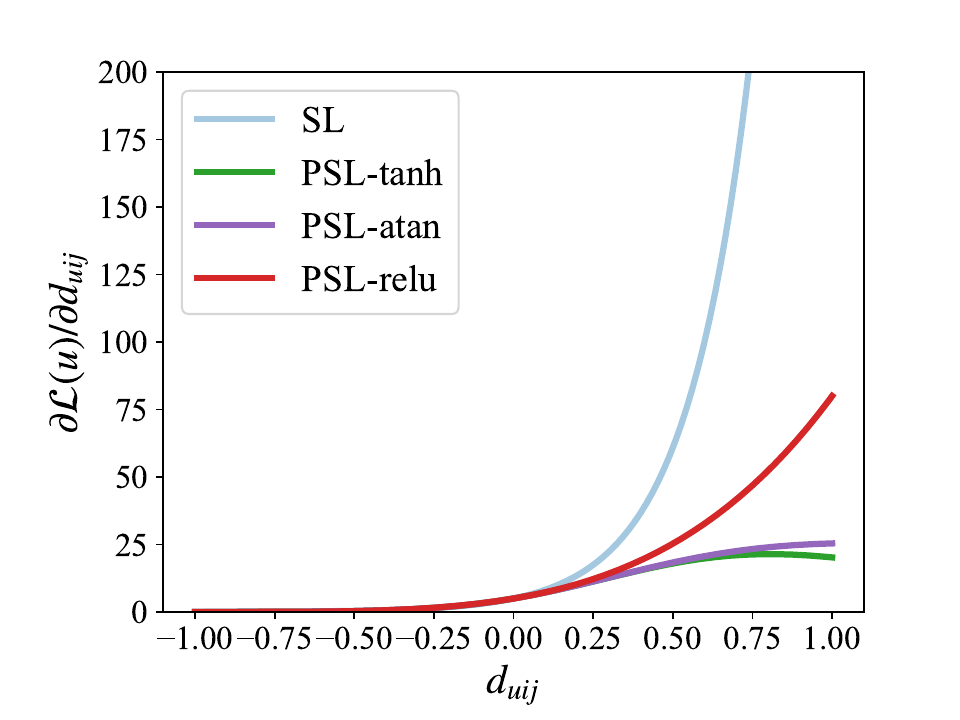}
        \caption{Weight distributions.}
        \label{fig:psl-activation-d}
    \end{subfigure}
    \caption{(a) Illustration of different surrogate activations. (b) The weight distribution of SL as compared with PSL using three different surrogate activations. Here we set $\tau=0.2$, which typically achieves optimal results in practice.}
\end{figure}

\textbf{Advantage 2: PSL controls the weight distribution.}  The gradient of PSL \wrt $d_{uij}$ is
\begin{equation} \label{eq:psl-gradient}
    \frac{\partial \mathcal{L}_{\textnormal{PSL}}(u)}{\partial d_{uij}}
    = \frac{\sigma'(d_{uij}) \cdot \sigma(d_{uij})^{1 / \tau - 1} / \tau}{|\mathcal{I}|\mathbb{E}_{j' \sim \mathcal{I}} [\sigma(d_{uij'})^{1/\tau}]}
    \textcolor{darkgreen}{\quad \propto \sigma'(d_{uij}) \cdot \sigma(d_{uij})^{1 / \tau - 1} / \tau}
\end{equation}
This implies that the shape of the weight distribution is determined by the choice of surrogate activation. By selecting appropriate activations, PSL can better balance the contributions of instances during training. For example, the three activations advocated before can explicitly mitigate the explosive issue on larger $d_{uij}$ (\cf \cref{fig:psl-activation-d}), bringing better robustness to false negative instances.

One might argue that adjusting $\tau$ in SL could improve noise resistance. However, such adjustments do not alter the fundamental shape of the weight distribution, which remains exponential. Furthermore, as we discuss subsequently, $\tau$ plays a crucial role in controlling robustness against distribution shifts. Thus, indiscriminate adjustments to $\tau$ may compromise out-of-distribution (OOD) robustness.

\textbf{Advantage 3: PSL is a DRO-empowered BPR loss.} We establish a connection between PSL and BPR \citep{rendle2009bpr} based on Distributionally Robust Optimization (DRO) \citep{shapiro2017distributionally,hu2013kullback}. Specifically, optimizing PSL is equivalent to applying a KL divergence DRO on negative item distribution over BPR loss (\cf \cref{eq:bpr}), as demonstrated in the following theorem\footnote{Note that $e^{\log(\sigma(d_{uij})) / \tau} = \sigma(d_{uij})^{1/\tau}$ holds, thus the PSL in \cref{eq:psl-dro-equivalent} is identical to the one in \cref{eq:psl}.}:

\begin{theorem} \label{thm:psl-lee-structure}
For each user $u$ and its positive item $i$, let $P = P(j | u, i)$ be the uniform distribution over $\mathcal{I}$. Given a robustness radius $\eta > 0$, consider the uncertainty set $\mathcal{Q}$ consisting of all perturbed distributions $Q = Q(j | u, i)$ satisfying: (i) $Q$ is absolutely continuous \wrt $P$, \ie $Q \ll P$; (ii) the KL divergence between $Q$ and $P$ is constrained by $\eta$, \ie $D_{\textnormal{KL}}(Q \| P) \leq \eta$. Then, optimizing PSL is equivalent to performing DRO over BPR loss, \ie
\begin{equation} \label{eq:psl-dro-equivalent}
    \min\underbrace{ \left\{ \mathbb{E}_{i \sim \mathcal{P}_u} \left[ \log \mathbb{E}_{j \sim \mathcal{I}} \left[e^{ \log(\sigma(d_{uij})) / \tau }\right] \vphantom{\sup_{Q \in \mathcal{Q}}} \right] \right\} }_{\mathcal{L}_{\textnormal{PSL}}(u)}
    \Leftrightarrow
    \min \underbrace{ \left\{ \mathbb{E}_{i \sim \mathcal{P}_u} \left[ \sup_{Q \in \mathcal{Q}} \mathbb{E}_{j \sim Q(j | u, i)} \left[ \log\sigma(d_{uij}) \right] \right] \right\} }_{\mathcal{L}_{\textnormal{BPR-DRO}}(u)}
\end{equation}
where $\tau = \tau(\eta)$ is a temperature parameter controlled by $\eta$.
\end{theorem}

The proof is presented in \cref{sec:appendix-proof-dro-lee-structure}. \cref{thm:psl-lee-structure} demonstrates how PSL, based on the DRO framework, is inherently robust to distribution shifts. This robustness is particularly valuable in RS, where user preference and item popularity may shift significantly. Therefore, PSL can be regarded as a robust generalization of BPR loss, offering better performance in OOD scenarios.

In addition, \cref{thm:psl-lee-structure} also gives insights into the \textbf{rationality of PSL} that differs from serving as a DCG surrogate loss, but rather as a DRO-empowered BPR loss:

\begin{itemize}[topsep=0pt,leftmargin=10pt]
    \setlength{\itemsep}{0pt}
    \item \textbf{Rationality of surrogate activations:} The activation function in BPR is originally chosen as an approximation to the Heaviside step function \citep{rendle2009bpr}. Since PSL is a generalization of BPR as stated in \cref{thm:psl-lee-structure}, it is reasonable to select the activations in PSL that aligns with the ones in BPR. Interestingly, this principle coincides with our analysis from the perspective of DCG surrogate loss.
    \item \textbf{Rationality of the position of temperature:} \cref{thm:psl-lee-structure} also rationalizes the extension form that places the temperature on the outside rather than inside. For the outside form (\ie $\sigma(d_{uij})^{1/\tau}$), \cref{thm:psl-lee-structure} holds, and the temperature $\tau$ can be interpreted as a Lagrange multiplier in DRO optimization, which controls the extent of distribution perturbation. However, for the inside form (\ie $\sigma(d_{uij} / \tau)$), \cref{thm:psl-lee-structure} no longer holds, and it would be challenging to establish the relationship between PSL and BPR.
    \item \textbf{Rationality of pairwise perspective:} Recent work such as BSL \citep{wu2023bsl} also reveals the DRO property of SL (\cf Lemma 1 in \citep{wu2023bsl}). However, we wish to highlight the distinctions between \cref{thm:psl-lee-structure} and \citet{wu2023bsl}'s analyses: 1) \citet{wu2023bsl} views SL from a pointwise perspective and associates it with a specific, less commonly used pointwise loss. In contrast, our analyses adopt a pairwise perspective and establish a relationship between PSL and the widely used BPR loss. 2) We construct a link between two families of losses with flexible activation selections, and \citet{wu2023bsl}'s analyses can be regarded as a special case within our broader framework.
\end{itemize}

The above analyses underscore the advantages of PSL and provide the principles to select surrogate activations. Remarkably, PSL is easily implemented and can be integrated into various recommendation scenarios. This can be achieved by merely replacing the exponential function $\exp(\cdot)$ in SL with another activation $\sigma(\cdot)$ surrogating the Heaviside step function, requiring minimal code modifications.

\subsection{Discussions} \label{sec:psl-discussion}

\textbf{Comparisons of two extension forms.} In previous discussions, we highlight the advantages of the form that positions the temperature outside (\ie $\sigma(d_{uij})^{1/\tau}$) over the inside (\ie $\sigma(d_{uij} / \tau)$). As discussed in the analyses of \cref{thm:psl-lee-structure}, the outside form can be regarded as a DRO-empowered BPR, while the inside form cannot, which ensures the robustness of PSL against distribution shifts.

Here we provide an additional perspective on the advantages of the outside form. In fact, the outside form facilitates the selection of surrogate activations. For instance, to ensure that PSL serves as a tighter DCG surrogate loss compared to SL (\ie ensure \cref{lemma:psl-dcg-surrogate} holds), the outside form only need to consider the condition \eqref{eq:psl-activation-condition} on the range of $d_{uij} \in [-1, 1]$. However, for the inside form, this condition should be satisfied on the entire domain of the activation $\sigma(\cdot)$, which complicates the selection of activation functions. Therefore, the outside form is more flexible and easier to implement. We further provide empirical evidence in \cref{sec:appendix-experiments-results-temperature-designs}, demonstrating that the inside form will lose the advantages of achieving tighter DCG surrogate loss, leading to compromised performance.

\textbf{Connections with other losses.} We further discuss the connections between PSL and other losses:

\begin{itemize}[topsep=0pt,leftmargin=10pt]
    \setlength{\itemsep}{0pt}
    \item \textbf{Connection with AdvInfoNCE \citep{zhang2024empowering}:} According to Theorem 3.1 in \citet{zhang2024empowering}, AdvInfoNCE can indeed be considered as a special case of PSL with $\sigma(\cdot)=\exp(\exp(\cdot))$. We argue that this activation is not a good choice as it would enlarge the gap between the loss and DCG. In fact, we have $-\log \mathrm{DCG} \leq \mathcal{L}_{\textnormal{PSL}} \leq \mathcal{L}_{\textnormal{SL}} \leq \mathcal{L}_{\textnormal{AdvInfoNCE}}$ (\cf \cref{sec:appendix-proof-all-dcg-surrogate} for proof). While AdvInfoNCE may achieve good performance in some specific OOD scenarios as tested in \citet{zhang2024empowering}, we argue that AdvInfoNCE is a looser DCG surrogate loss and would be highly sensitive to noise (\cf \cref{tab:iid-results,fig:ood-noise-results} in \cref{sec:experiments-iid} for empirical validation).
    \item \textbf{Connection with BPR \citep{rendle2009bpr}:} Besides the DRO relation stated in \cref{thm:psl-lee-structure}, we also derive the bound relation between BPR and PSL with the same activation, \ie $-\log \mathrm{DCG} \leq \mathcal{L}_{\textnormal{PSL}} \leq \log \mathcal{L}_{\textnormal{BPR}}$ (\cf \cref{sec:appendix-proof-all-dcg-surrogate} for proof). This relation clearly demonstrates the effectiveness of PSL over BPR --- performing DRO over BPR results robustness to distribution shifts, while also achieving a tighter surrogate of DCG, which is interesting (\cf \cref{tab:iid-results,tab:ood-shift-results} in \cref{sec:experiments-iid} for empirical validation). An intuitive explanation is that DCG focuses more on the higher-ranked items. Given that DRO would give more weight to the hard negative instances with larger prediction scores and higher positions, it would naturally narrow the gap between BPR and DCG.
\end{itemize}

\section{Experiments} \label{sec:experiments}

\subsection{Experimental Setup} \label{sec:experiments-setup}

\textbf{Testing scenarios.}
We adopt three representative testing scenarios to comprehensively evaluate model accuracy and robustness, including: 1) \textbf{IID setting:} the conventional testing scenario where training and test data are randomly split and identically distributed; 2) \textbf{OOD setting:} to assess the model's robustness on the out-of-distribution (OOD) data, we adopt a debiasing testing paradigm where the item popularity distribution shifts. We closely refer to \citet{zhang2023invariant}, \citet{wang2024distributionally}, and \citet{wei2021model}, sampling a test set where items are uniformly distributed while maintaining the long-tail nature of the training dataset; 3) \textbf{Noise setting:} to evaluate the model's sensitivity to noise, following \citet{wu2023bsl}, we manually impute a certain proportion of false negative items in the training data. The details of the above testing scenarios are provided in \cref{sec:appendix-experiments-datasets}.

\textbf{Datasets.}
Four widely-used datasets including Amazon-Book, Amazon-Electronic, Amazon-Movie \citep{he2016ups, mcauley2015image}, and Gowalla \citep{cho2011friendship} are used in our experiments. Considering the item popularity is not heavily skewed in the Amazon-Book and Amazon-Movie datasets, we turn to other conventional datasets, Amazon-CD \citep{he2016ups, mcauley2015image} and Yelp2018 \citep{yelp}, as replacements for OOD testing. All datasets are split into 80\% training set and 20\% test set, with 10\% of the training set further treated as the validation set. The details of the above datasets are summarized in \cref{sec:appendix-experiments-datasets}.

\textbf{Metrics.}
We closely refer to \citet{wu2023bsl} and \citet{zhang2024empowering}, adopting Top-$K$ metrics including $\mathrm{NDCG}@K$ \citep{jarvelin2017ir} and $\mathrm{Recall}@K$ \citep{fayyaz2020recommendation} for performance evaluation, where NDCG is the normalized DCG, \ie dividing DCG by the ideal value. Here we simply set $K = 20$ as in recent work \citep{wu2023bsl,zhang2024empowering} while observing similar results with other choices. For more details, please refer to \cref{sec:appendix-experiments-metrics}. 

\textbf{Compared methods.}
Five representative loss functions are compared in our experiments, including 1) the representative pairwise loss \textbf{BPR} (UAI'09 \citep{rendle2009bpr}); 2) the SOTA recommendation loss \textbf{Softmax Loss (SL)} (TOIS'24 \citep{wu2024effectiveness}) and its two DRO-enhancements \textbf{AdvInfoNCE} (NIPS'23 \citep{zhang2024empowering}) and \textbf{BSL} (ICDE'24 \citep{wu2023bsl}); 3) another SOTA loss \textbf{LLPAUC} (WWW'24 \citep{shi2024lower}) that optimizes the Lower-Left Partial AUC. Refer to \cref{sec:appendix-experiments-baselines} for more details about these baselines.

\textbf{Backbones.}
We also adopt three representative backbone models to evaluate the effectiveness of loss, including MF \citep{koren2009matrix}, LightGCN \citep{he2020lightgcn}, and XSimGCL \citep{yu2023xsimgcl}, see \cref{sec:appendix-experiments-backbones} for more details.

\textbf{Hyperparameter settings.}
A grid search is utilized to find the optimal hyperparameters. For all compared methods, we closely refer to the configurations provided in their respective publications to ensure their optimal performance. As we also carefully finetune SL, the improvements of existing methods over it are not as significant as those presented in their papers. The hyperparameter settings are provided in \cref{sec:appendix-experiments-hyperparameters}, where the detailed optimal hyperparameters for each method on each dataset and backbone are reported.

\begin{table}[t]
    \scriptsize
    \centering
    \caption{Performance comparison in terms of Recall@20 and NDCG@20 under the IID setting. The best result is bolded, and the blue-colored zone indicates that PSL is better than SL. Imp.\% denotes the NDCG@20 improvement of PSL over SL. The marker "*" indicates that the improvement is statistically significant ($p$-value $< 0.05$).}
    \begin{tabular}{c|l|cc|cc|cc|cc}
        \Xhline{1.2pt}
        \multirow{2}[4]{*}{\textbf{Model}} & \multicolumn{1}{c|}{\multirow{2}[4]{*}{\textbf{Loss}}} & \multicolumn{2}{c|}{\textbf{Amazon-Book}} & \multicolumn{2}{c|}{\textbf{Amazon-Electronic}} & \multicolumn{2}{c|}{\textbf{Amazon-Movie}} & \multicolumn{2}{c}{\textbf{Gowalla}} \bigstrut\\
        \cline{3-10}          &       & \textbf{Recall} & \textbf{NDCG} & \textbf{Recall} & \textbf{NDCG} & \textbf{Recall} & \textbf{NDCG} & \textbf{Recall} & \textbf{NDCG} \bigstrut\\
        \Xhline{1.0pt}
        \multirow{10}[4]{*}{MF \citep{koren2009matrix}} & BPR \citep{rendle2009bpr}   & 0.0665  & 0.0453  & 0.0816  & 0.0527  & 0.0916  & 0.0608  & 0.1355  & 0.1111  \bigstrut[t]\\
        & LLPAUC \citep{shi2024lower} & 0.1150  & 0.0811  & 0.0821  & 0.0499  & 0.1271  & 0.0883  & 0.1610  & 0.1189  \\
        & SL \citep{wu2024effectiveness}    & 0.1559  & 0.1210  & 0.0821  & 0.0529  & 0.1286  & 0.0929  & 0.2064  & 0.1624  \\
        & AdvInfoNCE \citep{zhang2024empowering} & 0.1557  & 0.1172  & 0.0829  & 0.0527  & 0.1293  & 0.0934  & 0.2067  & 0.1627  \\
        & BSL \citep{wu2023bsl}   & 0.1563  & 0.1212  & 0.0834  & 0.0530  & 0.1288  & 0.0931  & 0.2071  & 0.1630  \\
        & PSL-tanh & \cellcolor{lightblue}0.1567  & \cellcolor{lightblue}0.1225  & \cellcolor{lightblue}0.0832  & \cellcolor{lightblue}0.0535  & \cellcolor{lightblue}0.1297  & \cellcolor{lightblue}0.0941  & \cellcolor{lightblue}0.2088  & \cellcolor{lightblue}0.1646  \\
        & PSL-atan & \cellcolor{lightblue}0.1567  & \cellcolor{lightblue}0.1226  & \cellcolor{lightblue}0.0832  & \cellcolor{lightblue}0.0535  & \cellcolor{lightblue}0.1296  & \cellcolor{lightblue}0.0941  & \cellcolor{lightblue}0.2087  & \cellcolor{lightblue}0.1646  \\
        & PSL-relu & \cellcolor{lightblue}\textbf{0.1569} & \cellcolor{lightblue}\textbf{0.1227} & \cellcolor{lightblue}\textbf{0.0838} & \cellcolor{lightblue}\textbf{0.0541} & \cellcolor{lightblue}\textbf{0.1299} & \cellcolor{lightblue}\textbf{0.0945} & \cellcolor{lightblue}\textbf{0.2089} & \cellcolor{lightblue}\textbf{0.1647} \bigstrut[b]\\
        \cline{2-10}          & \textcolor{red}{\textbf{Imp.\%}} & \multicolumn{2}{c|}{\textcolor{red}{\textbf{+1.40\%*}}} & \multicolumn{2}{c|}{\textcolor{red}{\textbf{+2.31\%*}}} & \multicolumn{2}{c|}{\textcolor{red}{\textbf{+1.72\%*}}} & \multicolumn{2}{c}{\textcolor{red}{\textbf{+1.42\%*}}} \bigstrut\\
        \Xhline{1.0pt}
        \multirow{10}[4]{*}{LightGCN \citep{he2020lightgcn}} & BPR \citep{rendle2009bpr}   & 0.0984  & 0.0678  & 0.0813  & 0.0524  & 0.1006  & 0.0681  & 0.1745  & 0.1402  \bigstrut[t]\\
        & LLPAUC \citep{shi2024lower} & 0.1147  & 0.0810  & \textbf{0.0831} & 0.0507  & 0.1272  & 0.0886  & 0.1616  & 0.1192  \\
        & SL \citep{wu2024effectiveness}    & 0.1567  & 0.1220  & 0.0823  & 0.0526  & 0.1304  & 0.0941  & 0.2068  & 0.1628  \\
        & AdvInfoNCE \citep{zhang2024empowering} & 0.1568  & 0.1177  & 0.0823  & 0.0528  & 0.1292  & 0.0936  & 0.2066  & 0.1625  \\
        & BSL \citep{wu2023bsl}   & 0.1568  & 0.1220  & 0.0823  & 0.0526  & \textbf{0.1306} & 0.0943  & 0.2069  & 0.1628  \\
        & PSL-tanh & \cellcolor{lightblue}\textbf{0.1575} & \cellcolor{lightblue}\textbf{0.1233} & \cellcolor{lightblue}0.0825  & \cellcolor{lightblue}0.0532  & 0.1300  & \cellcolor{lightblue}0.0947  & \cellcolor{lightblue}\textbf{0.2091} & \cellcolor{lightblue}\textbf{0.1648} \\
        & PSL-atan & \cellcolor{lightblue}\textbf{0.1575} & \cellcolor{lightblue}\textbf{0.1233} & \cellcolor{lightblue}0.0825  & \cellcolor{lightblue}0.0532  & 0.1300  & \cellcolor{lightblue}0.0948  & \cellcolor{lightblue}\textbf{0.2091} & \cellcolor{lightblue}\textbf{0.1648} \\
        & PSL-relu & \cellcolor{lightblue}\textbf{0.1575} & \cellcolor{lightblue}\textbf{0.1233} & \cellcolor{lightblue}0.0830  & \cellcolor{lightblue}\textbf{0.0536} & 0.1300  & \cellcolor{lightblue}\textbf{0.0953} & \cellcolor{lightblue}0.2086  & \cellcolor{lightblue}\textbf{0.1648} \bigstrut[b]\\
        \cline{2-10}          & \textcolor{red}{\textbf{Imp.\%}} & \multicolumn{2}{c|}{\textcolor{red}{\textbf{+1.12\%*}}} & \multicolumn{2}{c|}{\textcolor{red}{\textbf{+1.98\%*}}} & \multicolumn{2}{c|}{\textcolor{red}{\textbf{+1.22\%*}}} & \multicolumn{2}{c}{\textcolor{red}{\textbf{+1.24\%*}}} \bigstrut\\
        \Xhline{1.0pt}
        \multirow{10}[4]{*}{XSimGCL \citep{yu2023xsimgcl}} & BPR \citep{rendle2009bpr}   & 0.1269  & 0.0905  & 0.0777  & \textbf{0.0508}  & 0.1236  & 0.0857  & 0.1966  & 0.1570  \bigstrut[t]\\
        & LLPAUC \citep{shi2024lower} & 0.1363  & 0.1008  & 0.0781  & 0.0481  & 0.1184  & 0.0828  & 0.1632  & 0.1200  \\
        & SL \citep{wu2024effectiveness}    & 0.1549  & 0.1207  & 0.0772  & 0.0490  & 0.1255  & 0.0905  & 0.2005  & 0.1570  \\
        & AdvInfoNCE \citep{zhang2024empowering} & 0.1568  & 0.1179  & 0.0776  & 0.0489  & 0.1252  & 0.0906  & 0.2010  & 0.1564  \\
        & BSL \citep{wu2023bsl}  & 0.1550  & 0.1207  & 0.0800  & 0.0507 & 0.1267  & 0.0918  & \textbf{0.2037} & \textbf{0.1597} \\
        & PSL-tanh & \cellcolor{lightblue}0.1567  & \cellcolor{lightblue}0.1225  & \cellcolor{lightblue}0.0790  & \cellcolor{lightblue}0.0501  & \cellcolor{lightblue}0.1308  & \cellcolor{lightblue}0.0926  & \cellcolor{lightblue}0.2034  & \cellcolor{lightblue}0.1591  \\
        & PSL-atan & \cellcolor{lightblue}0.1565  & \cellcolor{lightblue}0.1225  & \cellcolor{lightblue}0.0792  & \cellcolor{lightblue}0.0502  & 0.1253  & \cellcolor{lightblue}0.0917  & \cellcolor{lightblue}0.2035  & \cellcolor{lightblue}0.1591  \\
        & PSL-relu & \cellcolor{lightblue}\textbf{0.1571} & \cellcolor{lightblue}\textbf{0.1228} & \cellcolor{lightblue}\textbf{0.0801} & \cellcolor{lightblue}0.0507 & \cellcolor{lightblue}\textbf{0.1313} & \cellcolor{lightblue}\textbf{0.0935} & \cellcolor{lightblue}\textbf{0.2037} & \cellcolor{lightblue}0.1593  \bigstrut[b]\\
        \cline{2-10}          & \textcolor{red}{\textbf{Imp.\%}} & \multicolumn{2}{c|}{\textcolor{red}{\textbf{+1.72\%*}}} & \multicolumn{2}{c|}{\textcolor{red}{\textbf{+3.39\%*}}} & \multicolumn{2}{c|}{\textcolor{red}{\textbf{+3.42\%*}}} & \multicolumn{2}{c}{\textcolor{red}{\textbf{+1.48\%*}}} \bigstrut\\
        \Xhline{1.2pt}
    \end{tabular}
    \label{tab:iid-results}
\end{table}

\subsection{Performance Comparisons} \label{sec:experiments-iid}

\textbf{Results under IID setting.} \cref{tab:iid-results} presents the performance of our PSL compared with baselines.
\begin{itemize}[topsep=0pt,leftmargin=10pt]
    \setlength{\itemsep}{0pt}
    \item \textbf{PSL outperforms SL and other baselines.} Experimental results demonstrate that PSL, with three carefully selected surrogate activations, consistently outperforms SL across all datasets and backbones, with only a few exceptions. For instance, on the MF backbone, compared to the marginal improvements or sometimes even degradation of AdvInfoNCE (-3\%\textasciitilde 0.5\%) and BSL (0.0\%\textasciitilde 0.5\%), PSL shows a significant enhancement over SL (1\%\textasciitilde 3\%). Moreover, our PSL surpasses all compared baselines in most cases, clearly demonstrating its effectiveness.
    \item \textbf{PSL achieves tighter connections with ranking metrics.} We observe that the results align well with our theoretical analyses of PSL's Advantage 1 in \cref{sec:psl}. By replacing the exponential function with other suitable surrogate activations, PSL establishes a tighter relationship with ranking metrics, thus achieving better NDCG performance (\cf \cref{lemma:psl-dcg-surrogate}). This is also empirically evident from the larger improvements in NDCG compared to Recall. In contrast, as discussed in \cref{sec:psl-discussion}, other baselines like AdvInfoNCE and BSL either widen the gap or fail to connect with the ranking metrics, resulting in slight improvements or even performance drops.
\end{itemize}

\begin{table}[t]
    \scriptsize
    \centering
    \caption{Performance comparison in terms of Recall@20 and NDCG@20 under the OOD setting with popularity shift (on MF backbone). The best result is bolded, and the blue-colored zone indicates that PSL is better than SL. Imp.\% denotes the NDCG@20 improvement of PSL over SL. The marker "*" indicates that the improvement is statistically significant ($p$-value $< 0.05$).}
    \begin{tabular}{l|cc|cc|cc|cc}
        \Xhline{1.2pt}
        \multicolumn{1}{c|}{\multirow{2}[4]{*}{\textbf{Loss}}} & \multicolumn{2}{c|}{\textbf{Amazon-CD}} & \multicolumn{2}{c|}{\textbf{Amazon-Electronic}} & \multicolumn{2}{c|}{\textbf{Gowalla}} & \multicolumn{2}{c}{\textbf{Yelp2018}} \bigstrut\\
        \cline{2-9}          & \textbf{Recall} & \textbf{NDCG} & \textbf{Recall} & \textbf{NDCG} & \textbf{Recall} & \textbf{NDCG} & \textbf{Recall} & \textbf{NDCG} \bigstrut\\
        \Xhline{1.0pt}
        BPR \citep{rendle2009bpr}   & 0.0518  & 0.0318  & 0.0132  & 0.0069  & 0.0382  & 0.0273  & 0.0118  & 0.0072  \bigstrut[t]\\
        LLPAUC \citep{shi2024lower} & 0.1103  & 0.0764  & 0.0225  & 0.0134  & 0.0729  & 0.0522  & 0.0324  & 0.0210  \\
        SL \citep{wu2024effectiveness}   & 0.1184  & 0.0815  & 0.0230  & 0.0142  & 0.1006  & 0.0737  & 0.0349  & 0.0224  \\
        AdvInfoNCE \citep{zhang2024empowering} & 0.1189  & 0.0818  & 0.0228  & 0.0139  & 0.0927  & 0.0676  & 0.0348  & 0.0223  \\
        BSL \citep{wu2023bsl}  & 0.1184  & 0.0815  & 0.0231  & 0.0142  & 0.1006  & 0.0738  & 0.0351  & 0.0225  \\
        PSL-tanh & \cellcolor{lightblue}0.1202  & \cellcolor{lightblue}0.0834  & \cellcolor{lightblue}0.0239  & \cellcolor{lightblue}0.0146  & \cellcolor{lightblue}0.1013  & \cellcolor{lightblue}0.0748  & \cellcolor{lightblue}0.0357  & \cellcolor{lightblue}0.0228  \\
        PSL-atan & \cellcolor{lightblue}0.1202  & \cellcolor{lightblue}0.0835  & \cellcolor{lightblue}0.0239  & \cellcolor{lightblue}0.0146  & \cellcolor{lightblue}0.1013  & \cellcolor{lightblue}0.0748  & \cellcolor{lightblue}\textbf{0.0358} & \cellcolor{lightblue}0.0228  \\
        PSL-relu & \cellcolor{lightblue}\textbf{0.1203} & \cellcolor{lightblue}\textbf{0.0839} & \cellcolor{lightblue}\textbf{0.0241} & \cellcolor{lightblue}\textbf{0.0149} & \cellcolor{lightblue}\textbf{0.1014} & \cellcolor{lightblue}\textbf{0.0752} & \cellcolor{lightblue}\textbf{0.0358} & \cellcolor{lightblue}\textbf{0.0229} \bigstrut[b]\\
        \hline
        \textcolor{red}{\textbf{Imp.\%}} & \multicolumn{2}{c|}{\textcolor{red}{\textbf{+3.01\%*}}} & \multicolumn{2}{c|}{\textcolor{red}{\textbf{+5.02\%*}}} & \multicolumn{2}{c|}{\textcolor{red}{\textbf{+2.02\%*}}} & \multicolumn{2}{c}{\textcolor{red}{\textbf{+2.05\%*}}} \bigstrut\\
        \Xhline{1.2pt}
    \end{tabular}
    \label{tab:ood-shift-results}
\end{table}

\begin{figure}[t]
    \centering
    \begin{subfigure}[b]{0.325\textwidth}
        \includegraphics[width=\textwidth]{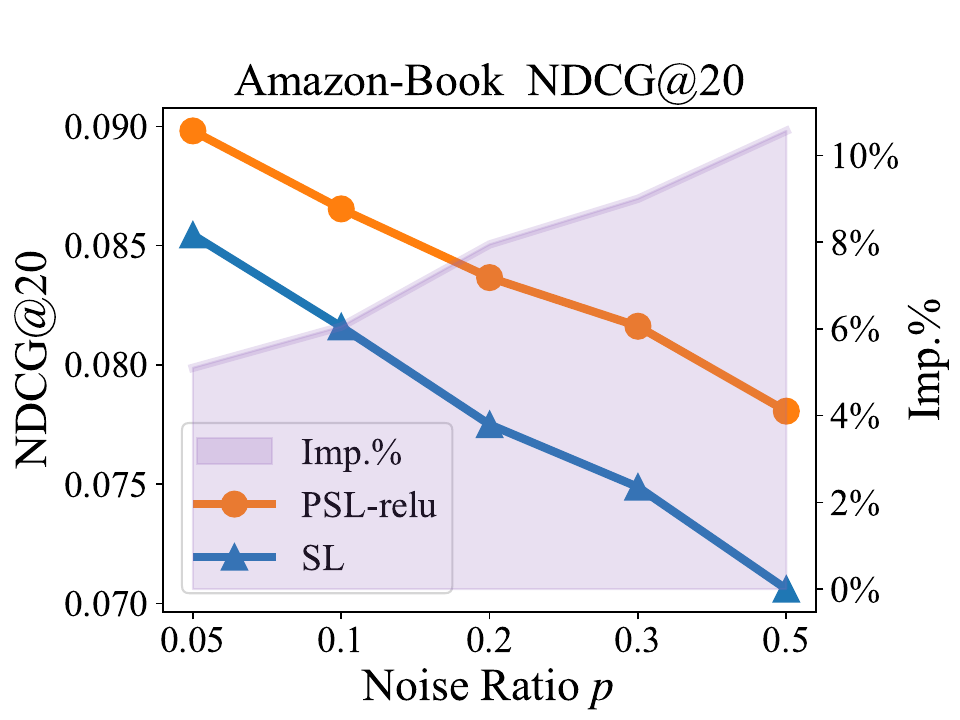}
        \caption{Amazon-Book}
    \end{subfigure}
    \begin{subfigure}[b]{0.325\textwidth}
        \includegraphics[width=\textwidth]{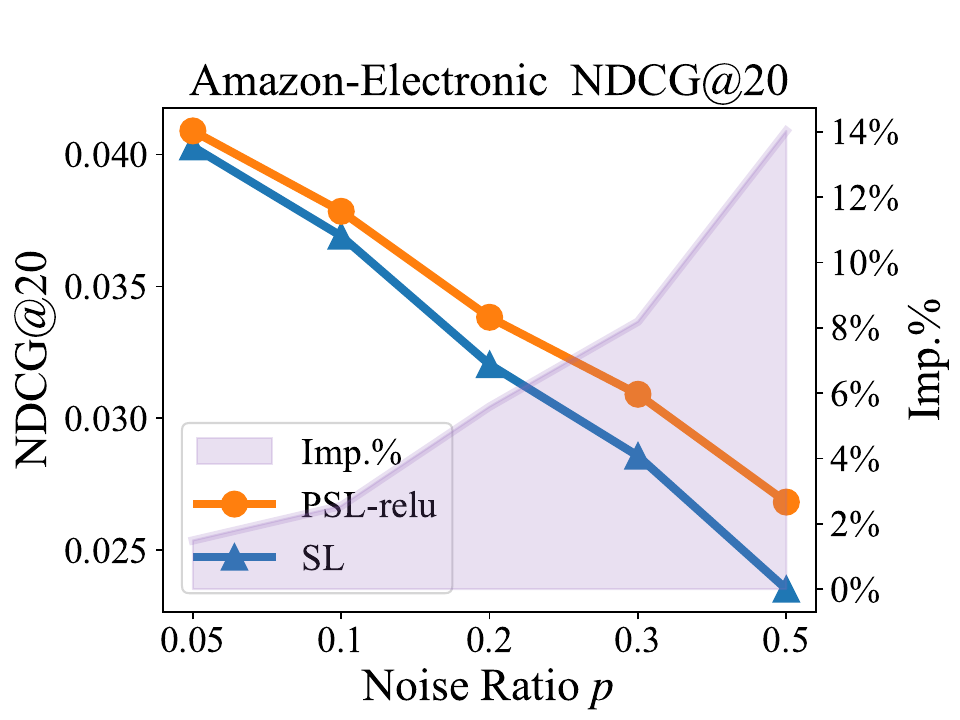}
        \caption{Amazon-Electronic}
    \end{subfigure}
    \begin{subfigure}[b]{0.325\textwidth}
        \includegraphics[width=\textwidth]{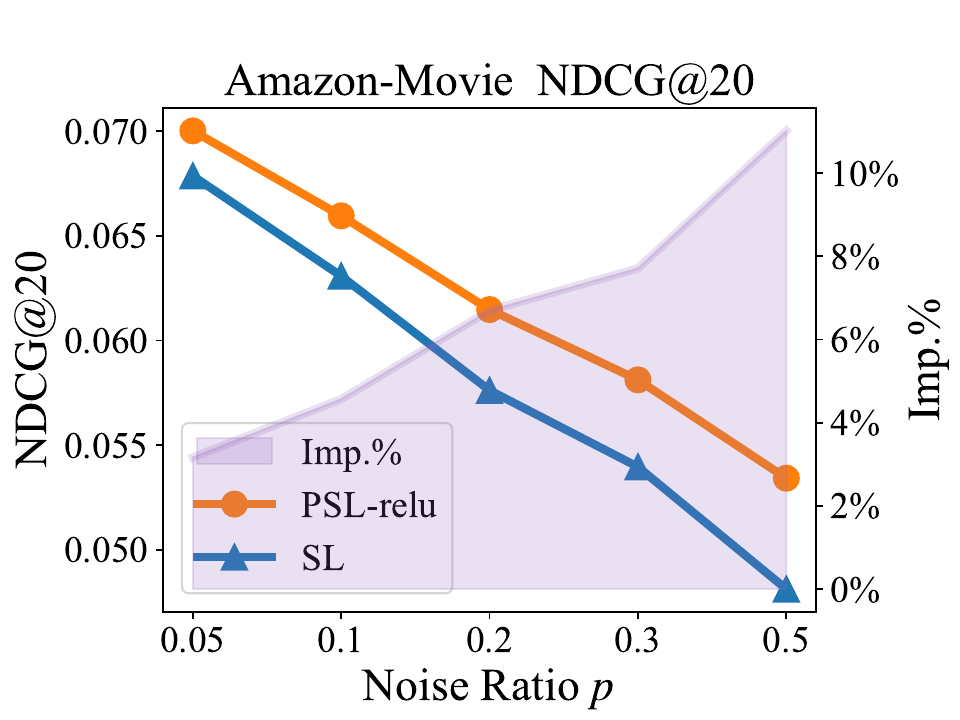}
        \caption{Amazon-Movie}
    \end{subfigure}
    \caption{Performance comparison of SL and PSL in terms of NDCG@20 with different false negative noise ratio (on MF backbone). We also present the relative improvements (\ie Imp.\%) achieved by PSL over SL. The complete results of other baselines are provided in \cref{sec:appendix-experiments-results-ood-noise}.} 
    \label{fig:ood-noise-results}
\end{figure}

\textbf{Results under OOD setting.} \cref{tab:ood-shift-results} presents the results in OOD scenarios with popularity shift. Given the consistent behavior across the three backbones, here we only report the results on MF.
\begin{itemize}[topsep=0pt,leftmargin=10pt]
    \setlength{\itemsep}{0pt}
    \item \textbf{PSL is robust to distribution shifts.} Experimental results indicate that PSL has a strong robustness against distribution shifts, which is consistent with PSL's Advantage 3 in \cref{sec:psl}. As can be seen, PSL not only outperforms all baselines (2\%\textasciitilde 5\%), but also achieves more pronounced improvements than in IID setting, like on Amazon-Electronic (2.31\% $\to$ 5.02\%) and Gowalla (1.42\% $\to$ 2.02\%). This demonstrates the superior robustness of PSL to distribution shifts, as shown in \cref{thm:psl-lee-structure}.
    \item \textbf{PSL is a DRO-enhancement of more reasonable loss.} Although both PSL and SL can be considered as DRO-enhanced losses (\cf \cref{thm:psl-lee-structure}), the original loss of our three PSLs before DRO-enhancement is more reasonable than that of SL, which degenerates from BPR loss to a linear triplet loss \citep{chen2017beyond}. Therefore, we observe significant improvements of PSL over SL.
\end{itemize}

\textbf{Results under Noise setting.}  \cref{fig:ood-noise-results,sec:appendix-experiments-results-ood-noise} presents the results with a certain ratio of imputed false negative noise. Specifically, we regard 10\% of the positive items in the training set as false negative noise and allow the negative sampling procedure to have a certain probability $p$ of sampling those items. We test the model performance with varying noise ratios $p \in \{0.05, 0.1, 0.2, 0.3, 0.5\}$. 
\begin{itemize}[topsep=0pt,leftmargin=10pt]
    \setlength{\itemsep}{0pt}
    \item \textbf{PSL has strong noise resistance.} Experimental results demonstrate that as the noise ratio $p$ increases, both the performance of SL and PSL decline. The performance decline rate of PSL is significantly smaller than that of other baselines, resulting in higher performance enhancement($> 10\%$ when $p=0.5$). These results indicate that PSL possesses stronger noise resistance than SL, which stems from our rational activation design, as discussed in PSL's Advantage 2 in \cref{sec:psl}. However, for DRO-enhanced losses such as AdvInfoNCE, the performance declines similarly to or even more quickly than SL (\cf \cref{sec:appendix-experiments-results-ood-noise}), which coincides with our theoretical analyses.
\end{itemize}

\section{Related Work} \label{sec:related-works}

\textbf{Model-related recommendation research.}
Recent years have witnessed flourishing publications on collaborative filtering (CF) models. The earliest works are mainly extensions of Matrix Factorization \citep{koren2009matrix}, building more complex interactions between embeddings \citep{fiesler2020handbook}, such as MF \citep{koren2009matrix}, LRML \citep{tay2018latent}, SVD \citep{deerwester1990indexing,bell2007modeling}, SVD++ \citep{koren2008factorization}, NCF \citep{he2017ncf}, \etc. In recent years, given the effectiveness of Graph Neural Networks (GNNs) \citep{wu2022graph,gao2022graph,kipf2016semi,wang2019neural,dong2021equivalence,wu2022graph_gcm,chen2024macro} in capturing  high-order relations, which align well with CF assumptions, GNN-based models have emerged and achieved great success, such as LightGCN \citep{he2020lightgcn}, NGCF \citep{wang2019neural}, LCF \citep{yu2020graph}, APDA \citep{zhou2023adaptive}, \etc. Building upon LightGCN, some works attempt to introduce contrastive learning \citep{liu2021self,oord2018representation} for graph data augmentation, such as SGL \citep{wu2021self} and XSimGCL \citep{yu2023xsimgcl}, achieving SOTA performance in recommendation.

\textbf{Loss-related recommendation research.} 
Existing recommendation losses can be primarily categorized into pointwise loss \citep{he2017ncf,he2017nfm}, pairwise loss \citep{rendle2009bpr}, and Softmax Loss (SL) \citep{wu2024effectiveness}, as discussed in \cref{sec:review-softmax}. Given the effectiveness of SL, recently some researchers have proposed to enhance SL from different perspectives. For instance, BSL \citep{wu2023bsl} aims to enhance the positive distribution robustness by leveraging Distributionally Robust Optimization (DRO); AdvInfoNCE \citep{zhang2024empowering} employs adversarial learning to enhance SL's robustness; \citet{zhang2023invariant} suggests incorporating bias-aware margins in SL to tackle popularity bias. Beyond these three types of losses, other approaches have also been explored in recent years. For example, \citet{zhao2021autoloss} introduces auto-loss, which utilizes automated machine learning techniques to search the optimal loss; \citet{shi2024lower} proposes LLPAUC to approximate Recall@$K$ metric. The main concerns with these losses are their lack of theoretical connections to ranking metrics like DCG, which may result in them not consistently outperforming the basic SL. Moreover, both auto-loss and LLPAUC require iterative learning, leading to additional computational time and increased instability.

\section{Conclusion and Limitations} \label{sec:conclusion}

In this work, we introduce a new family of loss functions, termed Pairwise Softmax Loss (PSL). PSL theoretically offers three advantages: 1) it serves as a better surrogate for ranking metrics with appropriate surrogate activations; 2) it allows flexible control over the distribution of the data contribution; 3) it can be interpreted as a specific BPR loss enhanced by Distributionally Robust Optimization (DRO). These properties demonstrate that PSL has greater effectiveness and robustness compared to Softmax Loss. Our extensive experiments across three testing scenarios validate the superiority of PSL over existing methods.

One limitation of both PSL and SL is inefficiency, as they require sampling a relatively large number of negative instances per iteration. How to address this issue and improve the efficiency of these losses is an interesting direction for future research.

\begin{ack}


This work is supported by the Zhejiang Province "JianBingLingYan+X" Research and Development Plan (2024C01114).

\end{ack}


\bibliographystyle{unsrtnat}
{\small
\bibliography{references}
}


\clearpage

\setcounter{figure}{0}
\setcounter{table}{0}
\renewcommand{\thefigure}{\thesection.\arabic{figure}}
\renewcommand{\thetable}{\thesection.\arabic{table}}

\appendix

\section{Theoretical Proofs} \label{sec:appendix-proof}

\subsection{Proof of \texorpdfstring{\cref{lemma:psl-dcg-surrogate}}{Lemma (PSL as DCG Surrogate)}} \label{sec:appendix-proof-psl-dcg-surrogate}

\begin{lemma}[\cref{lemma:psl-dcg-surrogate}]
If the condition
\begin{equation*} \tag{\ref{eq:psl-activation-condition}}
    \delta(d_{uij}) \leq \sigma(d_{uij}) \leq \exp(d_{uij})
\end{equation*}
is satisfied for any $d_{uij} \in [-1, 1]$, then PSL serves as a tighter DCG surrogate loss compared to SL.
\end{lemma}

\begin{proof}[Proof of \cref{lemma:psl-dcg-surrogate}]
For any $\tau > 0$, \cref{eq:psl-activation-condition} indicates that
\begin{equation} \label{eq:psl-activation-condition-proof-2}
    \delta(d_{uij}) \leq \sigma(d_{uij})^{1/\tau} \leq \exp(d_{uij})^{1/\tau}
\end{equation}
which means $\sigma(\cdot)^{1/\tau}$ is tighter than $\exp(\cdot)^{1/\tau}$ approximating $\delta(\cdot)$. According to \cref{eq:softmax-dcg-surrogate-bound,eq:softmax-dcg-pi} in \cref{sec:analysis-softmax}, we conclude that PSL is a tighter surrogate loss for DCG compared to SL.
\end{proof}

\subsection{Proof of \texorpdfstring{\cref{thm:psl-lee-structure}}{Theorem (PSL DRO under KL divergence)}} \label{sec:appendix-proof-dro-lee-structure}

\begin{theorem}[\cref{thm:psl-lee-structure}]
For each user $u$ and its positive item $i$, let $P = P(j | u, i)$ be the uniform distribution over $\mathcal{I}$. Given a robustness radius $\eta > 0$, consider the uncertainty set $\mathcal{Q}$ consisting of all perturbed distributions $Q = Q(j | u, i)$ satisfying: (i) $Q$ is absolutely continuous \wrt $P$, \ie $Q \ll P$; (ii) the KL divergence between $Q$ and $P$ is constrained by $\eta$, \ie $D_{\textnormal{KL}}(Q \| P) \leq \eta$. Then, optimizing PSL is equivalent to performing DRO over BPR loss, \ie
\begin{equation*} \tag{\ref{eq:psl-dro-equivalent}}
    \min\underbrace{ \left\{ \mathbb{E}_{i \sim \mathcal{P}_u} \left[ \log \mathbb{E}_{j \sim \mathcal{I}} \left[e^{ \log(\sigma(d_{uij})) / \tau }\right] \vphantom{\sup_{Q \in \mathcal{Q}}} \right] \right\} }_{\mathcal{L}_{\textnormal{PSL}}(u)}
    \Leftrightarrow
    \min \underbrace{ \left\{ \mathbb{E}_{i \sim \mathcal{P}_u} \left[ \sup_{Q \in \mathcal{Q}} \mathbb{E}_{j \sim Q(j | u, i)} \left[ \log\sigma(d_{uij}) \right] \right] \right\} }_{\mathcal{L}_{\textnormal{BPR-DRO}}(u)}
\end{equation*}
where $\tau = \tau(\eta)$ is a temperature parameter controlled by $\eta$.
\end{theorem}

To prove \cref{thm:psl-lee-structure}, it suffices to prove the following lemma:

\begin{lemma}[DRO under KL divergence] \label{lemma:dro-lee-structure}
Given the loss term $\ell(x; \theta)$ of input $x$ and parameters $\theta$, for any robustness radius $\eta > 0$, DRO under KL divergence is equivalent to optimizing a loss in the form of $\log \mathbb{E} [\exp(\cdot)]$, \ie
\begin{equation} \label{eq:dro-lee-structure}
    \min_{\theta} \sup_{Q \in \mathcal{Q}} \mathbb{E}_{x \sim Q} [\ell(x; \theta)]
    \Leftrightarrow
    \min_{\theta, \tau > 0} \left\{ \tau \log \mathbb{E}_{x \sim P} [\exp(\ell(x; \theta) / \tau)] + \tau \eta \right\}
\end{equation}
where the uncertainty set $\mathcal{Q}$ consists of all perturbed distributions $Q$ constrained by KL divergence \wrt the original distribution $P$, \ie $\mathcal{Q} = \{Q \ll P : D_{\textnormal{KL}}(Q \| P) \leq \eta\}$.
\end{lemma}

\cref{lemma:dro-lee-structure}, which was first proposed by \citet{hu2013kullback} with a complex proof, gives a closed-form solution for DRO under KL divergence. Here we provide an elegant proof based on the following general result about the $\phi$-divergence DRO, which was first proposed by \citet{shapiro2017distributionally}.

\begin{theorem}[DRO under $\phi$-divergence, \citep{shapiro2017distributionally}] \label{thm:dro-phi-divergence}
Consider the DRO problem in $\phi$-divergence
\begin{equation} \label{eq:phi-divergence}
    D_{\phi}(Q \| P) = \int \phi\left(\frac{\mathrm{d}Q}{\mathrm{d}P}\right) \mathrm{d}P
\end{equation}
where $\phi: \mathbb{R} \to \overline{\mathbb{R}}_+ = \mathbb{R}_+ \cup \{\infty\}$ is a convex function such that $\phi(1) = 0$ and $\phi(t) = +\infty$ for any $t < 0$. Then the inner maximization problem in DRO, \ie $\sup_{Q \in \mathcal{Q}} \mathbb{E}_{x \sim Q} [\ell(x; \theta)]$ with the uncertainty set $\mathcal{Q} = \{Q \ll P : D_{\phi}(Q \| P) \leq \eta\}$, is equivalent to the following optimization problem:
\begin{equation} \label{eq:dro-phi-divergence}
    \inf_{\tau > 0, \mu} \left\{ \mathbb{E}_{x \sim P}\left[(\tau \phi)^*(\ell(x; \theta) - \mu)\right] + \tau \eta + \mu \right\}
\end{equation}
where $f^*(y) = \sup_{x} \left\{yx - f(x)\right\}$ is the Fenchel conjugate \citep{rockafellar2009variational} for any convex function $f: \mathbb{R} \to \overline{\mathbb{R}}$.
\end{theorem}

\begin{proof}[Proof of \cref{thm:dro-phi-divergence}.]
Let the likelihood ratio $L(x) = \mathrm{d}Q(x) / \mathrm{d}P(x)$, then the inner maximization problem in DRO can be reformulated as
\begin{equation} \label{eq:dro-likelihood-ratio}
    \sup_{L \succeq 0} \left\{ \mathbb{E}_{x \sim P} \left[L(x) \ell(x; \theta)\right] \mid \mathbb{E}_{x \sim P}[\phi(L(x))] \leq \eta, \mathbb{E}_{x \sim P}[L(x)] = 1 \right\}
\end{equation}
The Lagrangian of \cref{eq:dro-likelihood-ratio} is
\begin{equation} \label{eq:dro-lagrangian}
    \mathcal{L}(L, \tau, \mu) = \mathbb{E}_{x \sim P} \left[ L(x) \ell(x; \theta) - \tau \phi(L(x)) - \mu L(x) \right] + \tau \eta + \mu
\end{equation}
where $\tau \geq 0$ and $\mu$ are the Lagrange multipliers. Problem \eqref{eq:dro-likelihood-ratio} is a convex optimization problem. One can easily check the Slater's condition \citep{boyd2004convex} by choosing $L(x) \equiv 1$, thus the strong duality \citep{boyd2004convex} holds, and problem \eqref{eq:dro-likelihood-ratio} is equivalent to the dual problem \eqref{eq:dro-likelihood-ratio-dual} of the Lagrangian \eqref{eq:dro-lagrangian}:
\begin{equation} \label{eq:dro-likelihood-ratio-dual}
    \inf_{\tau \geq 0, \mu} \sup_{L \succeq 0} \mathcal{L}(L, \tau, \mu)
\end{equation}
Consider the inner maximization problem $\sup_{L \succeq 0} \mathcal{L}(L, \tau, \mu)$ in \cref{eq:dro-likelihood-ratio-dual}, $\tau \eta + \mu$ is a constant and can be ignored. By the theorem of interchange of minimization and integration \citep{rockafellar2009variational}, we can interchange $\sup$ and expectation in \cref{eq:dro-likelihood-ratio-dual}.Then $\sup_{L \succeq 0} \mathcal{L}(L, \tau, \mu)$ can be reformulated as
\begin{equation} \label{eq:dro-likelihood-ratio-dual-solve}
    \mathbb{E}_{x \sim P} \left[ \sup_{L \succeq 0} \left\{ L(x) (\ell(x; \theta) - \mu) - \tau \phi(L(x)) \right\}\right]
\end{equation}
The above problem can be rewritten by the Fenchel conjugate as
\begin{equation} \label{eq:dro-likelihood-ratio-dual-solve-fenchel}
    \mathbb{E}_{x \sim P} \left[ (\tau \phi)^*(\ell(x; \theta) - \mu) \right]
\end{equation}
Thus, problem \eqref{eq:dro-likelihood-ratio-dual} is equivalent to
\begin{equation} \label{eq:dro-likelihood-ratio-dual-final}
    \inf_{\tau \geq 0, \mu} \left\{ \mathbb{E}_{x \sim P}\left[(\tau \phi)^*(\ell(x; \theta) - \mu)\right] + \tau \eta + \mu \right\}
\end{equation}
Finally, note that the condition $\tau \geq 0$ in problem \eqref{eq:dro-likelihood-ratio-dual-final} can be relaxed to $\tau > 0$ without affecting the optimal value, thus problem \eqref{eq:dro-likelihood-ratio-dual-final} is equivalent to problem \eqref{eq:dro-phi-divergence}, which completes the proof.
\end{proof}

\cref{lemma:dro-lee-structure} can be directly derived from \cref{thm:dro-phi-divergence} as follows:

\begin{proof}[Proof of \cref{lemma:dro-lee-structure}.]
KL divergence is a special case of $\phi$-divergence with $\phi(x) = x \log x$, and the Fenchel conjugate of $\tau \phi$ is
\begin{equation} \label{eq:kl-fenchel-conjugate}
    (\tau \phi)^*(y) = \sup_{x} \left\{yx - \tau x \log x\right\} = \tau e^{y / \tau - 1}
\end{equation}
By \cref{thm:dro-phi-divergence}, the DRO problem under KL divergence is equivalent to
\begin{equation} \label{eq:dro-KL-divergence}
\begin{aligned}
    &\inf_{\tau > 0, \mu} \left\{ \mathbb{E}_{x \sim P}\left[\tau e^{(\ell(x; \theta) - \mu) / \tau - 1}\right] + \tau \eta + \mu \right\} \\
    = &\inf_{\tau > 0, \mu} \left\{ \mathbb{E}_{x \sim P}\left[e^{\ell(x; \theta) / \tau} \right] \tau e^{-\mu / \tau - 1} + \tau \eta + \mu \right\}
\end{aligned}
\end{equation}
We fix $\tau$ and solve the optimal value of $\mu$ as
\begin{equation} \label{eq:dro-KL-divergence-mu}
    \mu^* = \tau \log \mathbb{E}_{x \sim P}\left[e^{\ell(x; \theta) / \tau}\right] - \tau
\end{equation}
Therefore, by substituting the optimal $\mu^*$ in \cref{eq:dro-KL-divergence-mu} back to \cref{eq:dro-KL-divergence}, the original DRO problem is equivalent to
\begin{equation} \label{eq:dro-KL-divergence-final}
    \inf_{\theta, \tau > 0} \left\{ \tau \log \mathbb{E}_{x \sim P}\left[e^{\ell(x; \theta) / \tau}\right] + \tau \eta \right\}
\end{equation}
This completes the proof.
\end{proof}

\cref{thm:psl-lee-structure} is a direct consequence of \cref{lemma:dro-lee-structure}, when setting the loss term $\ell(x; \theta)$ as $\log\sigma(d_{uij})$ (\ie the pairwise loss term in BPR loss), $P$ as the uniform distribution over $\mathcal{I}$, $Q$ as the perturbed distribution constrained by KL divergence \wrt $P$, and $\tau = \tau(\eta)$ as the optimal value of Lagrange multiplier $\tau$ in \cref{eq:dro-lee-structure}. This completes the proof of \cref{thm:psl-lee-structure}.

\subsection{Proof of the Bound Connections between PSL and Other Losses in \texorpdfstring{\cref{sec:psl-discussion}}{Theorem (DCG Surrogate Losses)}} \label{sec:appendix-proof-all-dcg-surrogate}

\begin{proof}[Proof of the Bound Connections in \cref{sec:psl-discussion}.]

We have proved in \cref{lemma:psl-dcg-surrogate} that
\begin{equation} \label{eq:softmax-dcg-final-proof-appendix}
	-\log \mathrm{DCG}(u) + \log|\mathcal{P}_u| 
	\leq \frac{1}{|\mathcal{P}_u|} \sum_{i \in \mathcal{P}_u} \log\left(\sum_{j \in \mathcal{I}} \sigma(d_{uij})^{{1}/{\tau}} \right)
\end{equation}
with any surrogate activation $\sigma$ satisfying $\delta(d_{uij}) \leq \sigma(d_{uij})$. Furthermore, if two surrogate activations $\sigma_1, \sigma_2$ satisfy $\sigma_1(d_{uij}) \leq \sigma_2(d_{uij})$ for any $d_{uij} \in [-1, 1]$, then the corresponding DCG surrogate losses satisfy the same inequality. Therefore, we have
\begin{equation} \label{eq:all-dcg-1}
    -\log \mathrm{DCG} 
    \leq \mathcal{L}_{\textnormal{PSL}}
    \leq \mathcal{L}_{\textnormal{SL}}
    \leq \mathcal{L}_{\textnormal{AdvInfoNCE}}
\end{equation}
where the constant term is omitted for simplicity.


Finally, we prove that BPR serves as a surrogate loss for DCG. Apply Jensen's inequality to the RHS of \cref{eq:softmax-dcg-final-proof-appendix}, we have
\begin{equation} \label{eq:bpr-dcg-proof}
    \frac{1}{|\mathcal{P}_u|} \sum_{i \in \mathcal{P}_u} \log\left(\sum_{j \in \mathcal{I}} \sigma(d_{uij})^{{1}/{\tau}} \right)
    \leq \log\left(\frac{1}{|\mathcal{P}_u|} \sum_{i \in \mathcal{P}_u} \sum_{j \in \mathcal{I}} \sigma(d_{uij})^{{1}/{\tau}} \right)
\end{equation}
The RHS of \cref{eq:bpr-dcg-proof} is just $\log \mathcal{L}_{\textnormal{BPR}}(u) - \log|\mathcal{P}_u|$ with the same surrogate activation $\sigma$ in BPR. \cref{eq:bpr-dcg-proof} indicates that for any surrogate activation $\sigma$, the general PSL (including SL, BSL, and AdvInfoNCE) is always better than BPR with the same $\sigma$, \ie
\begin{equation} \label{eq:bpr-dcg-proof-final}
    -\log \mathrm{DCG} 
    \leq \mathcal{L}_{\textnormal{PSL}}
    \leq \log \mathcal{L}_{\textnormal{BPR}}
\end{equation}
where the constant term is omitted for simplicity. This completes the proof.
\end{proof}

\clearpage
\section{Experimental Details} \label{sec:appendix-experiments}

\subsection{Datasets} \label{sec:appendix-experiments-datasets}

The six benchmark datasets used in our experiments are summarized in \cref{tab:dataset-statistics}. In dataset preprocessing, following the standard practice in \citet{wang2019neural}, we use 10-core setting \citep{he2016vbpr}, \ie all users and items have at least 10 interactions. We also remove the low-quality interactions, such as those with ratings (if available) lower than 3. After preprocessing, we split the datasets into 80\% training and 20\% test sets. In IID and Noise settings, we further randomly split a 10\% validation set from training set for hyperparameter tuning.

\begin{table}[htbp]
    \centering
    \caption{Statistics of datasets. All datasets are cleaned by 10-core setting. If the dataset is used in both IID and OOD settings, the statistics below are provided for the IID setting.}
    \begin{tabular}{l|rrrr}
        \Xhline{1.2pt}
        \multicolumn{1}{c|}{\textbf{Dataset}} & \multicolumn{1}{c}{\textbf{\#Users}} & \multicolumn{1}{c}{\textbf{\#Items}} & \multicolumn{1}{c}{\textbf{\#Interactions}} & \multicolumn{1}{c}{\textbf{Density}} \bigstrut\\
        \Xhline{1pt}
        Amazon-Electronic \citep{he2016ups, mcauley2015image} & 13,455  & 8,360  & 234,521  & 0.00208 \bigstrut[t]\\
        Amazon-CD \citep{he2016ups, mcauley2015image} & 12,784  & 13,874  & 360,763  & 0.00203 \\
        Amazon-Movie \citep{he2016ups, mcauley2015image} & 26,968  & 18,563  & 762,957  & 0.00152 \\
        Gowalla \citep{cho2011friendship} & 29,858  & 40,988  & 1,027,464  & 0.00084 \\
        Yelp2018 \citep{yelp} & 55,616  & 34,945  & 1,506,777  & 0.00078 \\
        Amazon-Book \citep{he2016ups, mcauley2015image} & 135,109  & 115,172  & 4,042,382  & 0.00026 \bigstrut[b]\\
        \Xhline{1.2pt}
    \end{tabular}
    \label{tab:dataset-statistics}
\end{table}

The details of datasets are as follows:
\begin{itemize}[topsep=0pt,leftmargin=10pt]
    \setlength{\itemsep}{0pt}
    \item \textbf{Amazon} \citep{he2016ups, mcauley2015image}: The Amazon dataset is a large crawl of product reviews from Amazon\footnote{\url{https://www.amazon.com/}}. The 2014 version of Amazon dataset contains 142.8 million reviews spanning May 1996 - July 2014. We process four widely-used categories: Electronic, CD, Movie, and Book, with interactions ranging from 200K to 4M. 
    \item \textbf{Gowalla} \citep{cho2011friendship}: The Gowalla dataset is a check-in dataset collected from the location-based social network Gowalla\footnote{\url{https://en.wikipedia.org/wiki/Gowalla}}, including 1M users, 1M locations, and 6M check-ins.
    \item \textbf{Yelp2018} \citep{yelp}: The Yelp\footnote{\url{https://www.yelp.com/}} dataset is a subset of Yelp's businesses, reviews, and user data, which was originally used in the Yelp Dataset Challenge. The 2018 version of Yelp dataset contains 5M reviews.
\end{itemize}

The detailed dataset constructions in IID, OOD and Noise settings are as follows:
\begin{itemize}[topsep=0pt,leftmargin=10pt]
    \setlength{\itemsep}{0pt}
    \item \textbf{IID setting} \citep{he2020lightgcn}: In the IID setting, the test set is randomly split from the original dataset. Specifically, the positive items of each user are split into 80\% training and 20\% test sets. Moreover, the training set is further split into 90\% training and 10\% validation sets for hyperparameter tuning. In the IID setting, the training and test sets are both long-tail.
    \item \textbf{OOD setting} \citep{zhang2023invariant,wang2024distributionally,wei2021model}: In the OOD setting, a 20\% test set is uniformly sampled (\wrt items) from the original dataset, while the 80\% training set remains long-tail. The OOD setting is used to simulate real-world online recommender systems. In order to avoid leaking information about the test set distribution, we do not introduce the validation set.
    \item \textbf{Noise setting} \citep{wu2023bsl}: In the Noise setting, the validation and test sets are split in the same way as the IID setting. However, we randomly sample 10\% of the training set as the \emph{false negatives}. In Noise training, the negative items will be sampled from the false negatives with a probability of $p$ as the \emph{negative noise}, where $p \in \{0.05, 0.1, 0.2, 0.3, 0.5\}$ is \aka the \emph{noise ratio}.
\end{itemize}

All experiments are conducted on one NVIDIA GeForce RTX 4090 GPU and one AMD EPYC 7763 64-Core Processor.

\subsection{Metrics} \label{sec:appendix-experiments-metrics}

This section provides a detailed explanation of the recommendation metrics used or mentioned in our experiments. 

As stated in \cref{sec:experiments-setup}, we use Top-$K$ recommendation \citep{liu2009learning}. It should be noted that for each user, the positive items in the training set will be masked and not included in the Top-$K$ recommendations when evaluating, and the ground-truth positive items $\mathcal{P}_u$ only consist of those in the test set. For convenience, we denote the set of hit items in the Top-$K$ recommendations for user $u$ as $\mathcal{H}_u = \{ i \in \mathcal{P}_u : \pi_u(i) \leq K\}$. The recommendation metrics are defined as follows:

\begin{itemize}[topsep=0pt,leftmargin=10pt]
    \setlength{\itemsep}{0pt}
    \item $\bm{\mathrm{Recall}@K}$ \citep{fayyaz2020recommendation}: The proportion of hit items among $\mathcal{P}_u$ in the Top-$K$ recommendations, \ie $\mathrm{Recall}@K(u) = |\mathcal{H}_u| / |\mathcal{P}_u|$, and the overall $\mathrm{Recall}@K = \mathbb{E}_{u \sim \mathcal{U}} [\mathrm{Recall}@K(u)]$.
    \item $\bm{\mathrm{NDCG}@K}$ \citep{jarvelin2017ir}: The Discounted Cumulative Gain in the Top-$K$ recommendations ($\mathrm{DCG}@K$) is defined as $\mathrm{DCG}@K(u) = \sum_{i \in \mathcal{H}_u} 1 / \log_2(1 + \pi_u(i))$. Since the range of $\mathrm{DCG}@K$ will vary with the number of positive items $|\mathcal{P}_u|$, we should consider to normalize $\mathrm{DCG}@K$ to $[0, 1]$. The Normalized DCG in the Top-$K$ recommendations ($\mathrm{NDCG}@K) = \mathrm{DCG@}K(u) / \mathrm{IDCG}@K(u)$, where $\mathrm{IDCG}@K$ is the ideal $\mathrm{DCG}@K$, \ie $\mathrm{IDCG}@K(u) = \sum_{i = 1}^{\min\{K, |\mathcal{P}_u|\}} 1 / \log_2(1 + i)$. The overall $\mathrm{NDCG}@K = \mathbb{E}_{u \sim \mathcal{U}} [\mathrm{NDCG}@K(u)]$.
    \item $\bm{\mathrm{MRR}@K}$ \citep{lu2023optimizing, argyriou2020microsoft}: The Mean Reciprocal Rank (MRR) is originally defined as the reciprocal of the rank of the first hit item. Here we follow the definition of \citet{argyriou2020microsoft}'s to meet the requirements of multi-hit scenarios, \ie $\mathrm{MRR}@K(u) = \mathbb{E}_{i \sim \mathcal{H}_u} [1 / \pi_u(i)]$, and the overall $\mathrm{MRR}@K = \mathbb{E}_{u \sim \mathcal{U}} [\mathrm{MRR}@K(u)]$.
\end{itemize}

\subsection{Baselines} \label{sec:appendix-experiments-baselines}

We reproduced the following losses as baselines in our experiments:

\begin{itemize}[topsep=0pt,leftmargin=10pt]
    \setlength{\itemsep}{0pt}
    \item \textbf{BPR} \citep{rendle2009bpr}: A pairwise loss based on the Bayesian Maximum Likelihood Estimation (MLE). The objective of BPR is to learn a partial order among items, \ie positive items should be ranked higher than negative items. Furthermore, BPR is a surrogate loss for AUC metric \citep{rendle2009bpr, silveira2019good}. In our implementation, we follow \citet{he2020lightgcn}'s setting and use the inner product as the similarity function for user and item embeddings.
    \item \textbf{LLPAUC} \citep{shi2024lower}: A surrogate loss for Recall and Precision. In fact, LLPAUC is a surrogate loss for the lower-left part of AUC. In practice, LLPAUC is a min-max loss.
    \item \textbf{Softmax Loss (SL)} \citep{wu2024effectiveness}: A SOTA recommendation loss derived from the listwise MLE, \ie maximizing the probability of the positive items among all items. The effectiveness of SL has been thoroughly reviewed in \cref{sec:review-softmax,sec:analysis-softmax}. In fact, SL is a special case of PSL with surrogate activation $\sigma = \exp(\cdot)$.
    \item \textbf{AdvInfoNCE} \citep{zhang2024empowering}: A DRO-based modification of SL. AdvInfoNCE tries to introduce adaptive negative hardness to pairwise score $d_{uij}$ in SL (\cf \cref{eq:softmax-pairwise}). In \citet{zhang2024empowering}'s original design, AdvInfoNCE can be seen as a failure case of PSL with surrogate activation $\sigma = \exp(\exp(\cdot))$, as discussed in \cref{sec:psl-discussion}. In practice, AdvInfoNCE is a min-max loss.
    \item \textbf{BSL} \citep{wu2023bsl}: A DRO-based modification of SL. BSL applies additional DRO on the positive term in the pointwise form of SL.
\end{itemize}

The hyperparameter settings of each method are detailed in \cref{sec:appendix-experiments-hyperparameters}.

\subsection{Backbones} \label{sec:appendix-experiments-backbones}

We implemented three popular recommendation backbones in our experiments, including

\begin{itemize}[topsep=0pt,leftmargin=10pt]
    \setlength{\itemsep}{0pt}
    \item \textbf{MF} \citep{koren2009matrix}: MF is the most basic but still effective recommendation model, which factorizes the user-item interaction matrix into user and item embeddings. All the embedding-based recommendation models use MF as the first layer. Specifically, we set the embedding size $d = 64$ for all settings, following the setting in \citet{wang2019neural}.
    \item \textbf{LightGCN} \citep{he2020lightgcn}: LightGCN is an effective GNN-based recommendation model. LightGCN performs graph convolution on the user-item interaction graph, so as to aggregate the high-order interactions. Specifically, LightGCN simplifies NGCF \citep{wang2019neural} and only retains the non-parameterized graph convolution operator. In our experiments, we set the number of layers as 2, which aligns with the original setting in \citet{he2020lightgcn}.
    \item \textbf{XSimGCL} \citep{yu2023xsimgcl}: XSimGCL is a novel recommendation model based on contrastive learning \citep{liu2021self,jaiswal2020survey}. Based on 3-layers LightGCN, XSimGCL adds a random noise to the output embeddings of each layer, and introduces the contrastive learning between the final layer and the $l^*$-th layer, \ie adding an auxiliary InfoNCE loss \citep{oord2018representation} between these two layers. Following the original \citet{yu2023xsimgcl}'s setting, the modulus of random noise between each layer is set as 0.1, the contrastive layer $l^* = 1$ (where the embedding layer is 0-th layer), the temperature of InfoNCE is set as 0.1, and the weight of the auxiliary InfoNCE loss is set as 0.2 (except for the Amazon-Electronic dataset, where the weight is set as 0.05).
\end{itemize}

\subsection{Hyperparameters} \label{sec:appendix-experiments-hyperparameters}

\subsubsection{Hyperparameter Settings} \label{sec:appendix-experiments-hyperparameters-settings}

\textbf{Optimizer.}
We use Adam \citep{kingma2014adam} optimizer for training. The learning rate (lr) is searched in $\{10^{-1}, 10^{-2}, 10^{-3}\}$, except for BPR, where the lr is searched in $\{10^{-1}, 10^{-2}, 10^{-3}, 10^{-4}\}$. The weight decay (wd) is searched in $\{0, 10^{-4}, 10^{-5}, 10^{-6}\}$. The batch size is set as 1024, and the number of epochs is set as 200. Following the negative sampling strategy in \citet{wu2023bsl}, we uniformly sample 1000 negative items for each positive instance in training.

\textbf{Loss.}
The hyperparameters of each loss are detailed as follows:
\begin{itemize}[topsep=0pt,leftmargin=10pt]
    \setlength{\itemsep}{0pt}
    \item \textbf{BPR}: No other hyperparameters. 
    \item \textbf{LLPAUC}: Following \citet{shi2024lower}'s setting, the hyperparameters $\alpha \in \{0.1, 0.3, 0.5, 0.7, 0.9\}$ and $\beta \in \{0.01, 0.1\}$ are searched.
    \item \textbf{Softmax Loss (SL)}: The temperature $\tau \in \{0.005, 0.025, 0.05, 0.1, 0.25\}$ is searched.
    \item \textbf{AdvInfoNCE}: The temperature $\tau$ is searched in the same space as SL. The other hyperparameters are fixed as the original setting in \citet{zhang2024empowering}. Specifically, the negative weight is set as $64$, the adversarial learning will be performed every 5 epochs, with the adversarial learning rate as $5 \times 10^{-5}$. 
    \item \textbf{BSL}: The temperatures $\tau_1, \tau_2$ for positive and negative terms are searched in the same space as SL, respectively.
    \item \textbf{PSL}: The temperature $\tau$ is searched in the same space as SL. 
\end{itemize}

\subsubsection{Optimal Hyperparameters} \label{sec:appendix-experiments-best-hyperparameters}

The hyperparameters we search include the learning rate (lr), weight decay (wd), and other hyperparameters: $\{\alpha, \beta\}$ for LLPAUC, $\{\tau\}$ for SL, AdvInfoNCE, and PSL, $\{\tau_1, \tau_2\}$ for BSL. 

\textbf{IID optimal hyperparameters.} \cref{tab:appendix-iid-hyperparameters} shows the optimal hyperparameters of IID setting, including four datasets (Amazon-Book, Amazon-Electronic, Amazon-Movie, Gowalla) and three backbones (MF, LightGCN, XSimGCL).

\textbf{OOD optimal hyperparameters.} \cref{tab:appendix-ood-shift-hyperparameters} shows the optimal hyperparameters of OOD setting on MF backbone, including four datasets (Amazon-CD, Amazon-Electronic, Gowalla, Yelp2018).

\textbf{Noise optimal hyperparameters.} The Noise setting uses the optimal hyperparameters of IID setting, as listed in \cref{tab:appendix-iid-hyperparameters}. We compare the performance of each method under different noise ratios $p \in \{0.05, 0.1, 0.2, 0.3, 0.5\}$ on MF backbone and four IID datasets (Amazon-Book, Amazon-Electronic, Amazon-Movie, Gowalla).

\begin{table}[htbp]
    \scriptsize
    \centering
    \caption{Optimal hyperparameters of IID setting.}
    \begin{tabular}{c|l|ccc|ccc}
    \Xhline{1.2pt}
    \multirow{2}[4]{*}{\textbf{Model}} & \multicolumn{1}{c|}{\multirow{2}[4]{*}{\textbf{Loss}}} & \multicolumn{3}{c|}{\textbf{Amazon-Book}} & \multicolumn{3}{c}{\textbf{Amazon-Electronic}} \bigstrut\\
    \cline{3-8}          &       & \textbf{lr} & \textbf{wd} & \textbf{others} & \textbf{lr} & \textbf{wd} & \textbf{others} \bigstrut\\
    \Xhline{1.0pt}
    \multirow{9}[2]{*}{MF} & BPR   & $10^{-4}$ & 0 &       & $10^{-3}$ & $10^{-5}$     &  \bigstrut[t]\\
    & LLPAUC & $10^{-1}$ & 0     & \{0.7, 0.01\} & $10^{-1}$ & 0     & \{0.5, 0.01\} \\
    & AdvInfoNCE & $10^{-2}$ & 0     & \{0.05\} & $10^{-1}$ & 0     & \{0.1\} \\
    & SL & $10^{-1}$ & 0     & \{0.025\} & $10^{-2}$ & 0     & \{0.1\} \\
    & BSL   & $10^{-1}$ & 0     & \{0.25, 0.025\} & $10^{-1}$ & 0     & \{0.25, 0.1\} \\
    & PSL-tanh & $10^{-1}$ & 0     & \{0.025\} & $10^{-2}$ & 0     & \{0.1\} \\
    & PSL-atan & $10^{-1}$ & 0     & \{0.025\} & $10^{-2}$ & 0     & \{0.1\} \\
    & PSL-relu & $10^{-1}$ & 0     & \{0.025\} & $10^{-2}$ & 0     & \{0.1\} \bigstrut[b]\\
    \Xhline{1.0pt}
    \multirow{9}[2]{*}{LightGCN} & BPR   & $10^{-3}$ & 0 &       & $10^{-2}$ & $10^{-6}$     &  \bigstrut[t]\\
    & LLPAUC & $10^{-1}$ & 0     & \{0.7, 0.01\} & $10^{-1}$ & 0     & \{0.5, 0.01\} \\
    & AdvInfoNCE & $10^{-1}$ & 0     & \{0.05\} & $10^{-2}$ & 0     & \{0.1\} \\
    & SL & $10^{-1}$ & 0     & \{0.025\} & $10^{-2}$ & 0     & \{0.1\} \\
    & BSL   & $10^{-1}$ & 0     & \{0.25, 0.025\} & $10^{-2}$ & 0     & \{0.1, 0.1\} \\
    & PSL-tanh & $10^{-1}$ & 0     & \{0.025\} & $10^{-2}$ & 0     & \{0.1\} \\
    & PSL-atan & $10^{-1}$ & 0     & \{0.025\} & $10^{-2}$ & 0     & \{0.1\} \\
    & PSL-relu & $10^{-1}$ & 0     & \{0.025\} & $10^{-2}$ & 0     & \{0.1\} \bigstrut[b]\\
    \Xhline{1.0pt}
    \multirow{9}[2]{*}{XSimGCL} & BPR   & $10^{-4}$ & $10^{-5}$ &       & $10^{-2}$ & 0     &  \bigstrut[t]\\
    & LLPAUC & $10^{-1}$ & 0     & \{0.7, 0.01\} & $10^{-1}$ & 0     & \{0.3, 0.01\} \\
    & AdvInfoNCE & $10^{-1}$ & 0     & \{0.05\} & $10^{-1}$ & 0     & \{0.1\} \\
    & SL & $10^{-1}$ & 0     & \{0.025\} & $10^{-2}$ & 0     & \{0.1\} \\
    & BSL   & $10^{-1}$ & 0     & \{0.025, 0.025\} & $10^{-1}$ & 0     & \{0.05, 0.1\} \\
    & PSL-tanh & $10^{-2}$ & 0     & \{0.025\} & $10^{-1}$ & 0     & \{0.1\} \\
    & PSL-atan & $10^{-2}$ & 0     & \{0.025\} & $10^{-1}$ & 0     & \{0.1\} \\
    & PSL-relu & $10^{-1}$ & 0     & \{0.025\} & $10^{-1}$ & 0     & \{0.1\} \bigstrut[b]\\
    \Xhline{1.2pt}
    \multirow{2}[4]{*}{\textbf{Model}} & \multicolumn{1}{c|}{\multirow{2}[4]{*}{\textbf{Loss}}} & \multicolumn{3}{c|}{\textbf{Amazon-Movie}} & \multicolumn{3}{c}{\textbf{Gowalla}} \bigstrut\\
    \cline{3-8}          &       & \textbf{lr} & \textbf{wd} & \textbf{others} & \textbf{lr} & \textbf{wd} & \textbf{others} \bigstrut\\
    \Xhline{1.0pt}
    \multirow{9}[2]{*}{MF} & BPR   & $10^{-3}$ & $10^{-6}$ &       & $10^{-3}$ & $10^{-6}$     &  \bigstrut[t]\\
    & LLPAUC & $10^{-1}$ & 0     & \{0.7, 0.01\} & $10^{-1}$ & 0     & \{0.7, 0.01\} \\
    & AdvInfoNCE & $10^{-1}$ & 0     & \{0.05\} & $10^{-1}$ & 0     & \{0.05\} \\
    & SL & $10^{-1}$ & 0     & \{0.05\} & $10^{-1}$ & 0     & \{0.05\} \\
    & BSL   & $10^{-2}$ & 0     & \{0.25, 0.05\} & $10^{-1}$ & 0     & \{0.1, 0.05\} \\
    & PSL-tanh & $10^{-1}$ & 0     & \{0.05\} & $10^{-1}$ & 0     & \{0.05\} \\
    & PSL-atan & $10^{-1}$ & 0     & \{0.05\} & $10^{-1}$ & 0     & \{0.05\} \\
    & PSL-relu & $10^{-1}$ & 0     & \{0.05\} & $10^{-1}$ & 0     & \{0.05\} \bigstrut[b]\\
    \Xhline{1.0pt}
    \multirow{9}[2]{*}{LightGCN} & BPR   & $10^{-3}$ & 0 &       & $10^{-3}$ & 0     &  \bigstrut[t]\\
    & LLPAUC & $10^{-1}$ & 0     & \{0.7, 0.01\} & $10^{-1}$ & 0     & \{0.7, 0.01\} \\
    & AdvInfoNCE & $10^{-1}$ & 0     & \{0.05\} & $10^{-1}$ & 0     & \{0.05\} \\
    & SL & $10^{-1}$ & 0     & \{0.05\} & $10^{-1}$ & 0     & \{0.05\} \\
    & BSL   & $10^{-1}$ & 0     & \{0.025, 0.05\} & $10^{-1}$ & 0     & \{0.025, 0.05\} \\
    & PSL-tanh & $10^{-1}$ & 0     & \{0.05\} & $10^{-1}$ & 0     & \{0.05\} \\
    & PSL-atan & $10^{-1}$ & 0     & \{0.05\} & $10^{-1}$ & 0     & \{0.05\} \\
    & PSL-relu & $10^{-1}$ & 0     & \{0.05\} & $10^{-1}$ & 0     & \{0.05\} \bigstrut[b]\\
    \Xhline{1.0pt}
    \multirow{9}[2]{*}{XSimGCL} & BPR   & $10^{-4}$ & $10^{-4}$ &       & $10^{-4}$ & 0     &  \bigstrut[t]\\
    & LLPAUC & $10^{-1}$ & 0     & \{0.3, 0.01\} & $10^{-1}$ & 0     & \{0.7, 0.01\} \\
    & AdvInfoNCE & $10^{-1}$ & 0     & \{0.05\} & $10^{-1}$ & 0     & \{0.05\} \\
    & SL & $10^{-2}$ & 0     & \{0.05\} & $10^{-2}$ & 0     & \{0.05\} \\
    & BSL   & $10^{-1}$ & 0     & \{0.025, 0.05\} & $10^{-1}$ & 0     & \{0.025, 0.05\} \\
    & PSL-tanh & $10^{-1}$ & 0     & \{0.1\} & $10^{-1}$ & 0     & \{0.05\} \\
    & PSL-atan & $10^{-1}$ & 0     & \{0.05\} & $10^{-1}$ & 0     & \{0.05\} \\
    & PSL-relu & $10^{-2}$ & 0     & \{0.1\} & $10^{-1}$ & 0     & \{0.05\} \bigstrut[b]\\
    \Xhline{1.2pt}
    \end{tabular}
    \label{tab:appendix-iid-hyperparameters}
\end{table}

\begin{table}[htbp]
    \scriptsize
    \centering
    \caption{Optimal hyperparameters of OOD setting.}
    \begin{tabular}{c|l|ccc|ccc}
    \Xhline{1.2pt}
    \multirow{2}[4]{*}{\textbf{Model}} & \multicolumn{1}{c|}{\multirow{2}[4]{*}{\textbf{Loss}}} & \multicolumn{3}{c|}{\textbf{Amazon-CD}} & \multicolumn{3}{c}{\textbf{Amazon-Electronic}} \bigstrut\\
    \cline{3-8}          &       & \textbf{lr} & \textbf{wd} & \textbf{others} & \textbf{lr} & \textbf{wd} & \textbf{others} \bigstrut\\
    \Xhline{1.0pt}
    \multirow{9}[2]{*}{MF} & BPR   & $10^{-2}$ & $10^{-6}$ &       & $10^{-2}$ & $10^{-6}$     &  \bigstrut[t]\\
    & LLPAUC & $10^{-1}$ & 0     & \{0.7, 0.01\} & $10^{-1}$ & 0     & \{0.7, 0.1\} \\
    & AdvInfoNCE & $10^{-1}$ & 0     & \{0.05\} & $10^{-1}$ & 0     & \{0.05\} \\
    & SL & $10^{-1}$ & 0     & \{0.05\} & $10^{-1}$ & 0     & \{0.05\} \\
    & BSL   & $10^{-1}$ & 0     & \{0.05, 0.05\} & $10^{-1}$ & 0     & \{0.1, 0.05\} \\
    & PSL-tanh & $10^{-1}$ & 0     & \{0.05\} & $10^{-1}$ & 0     & \{0.05\} \\
    & PSL-atan & $10^{-1}$ & 0     & \{0.05\} & $10^{-1}$ & 0     & \{0.05\} \\
    & PSL-relu & $10^{-1}$ & 0     & \{0.05\} & $10^{-1}$ & 0     & \{0.05\} \bigstrut[b]\\
    \Xhline{1.2pt}
    \multirow{2}[4]{*}{\textbf{Model}} & \multicolumn{1}{c|}{\multirow{2}[4]{*}{\textbf{Loss}}} & \multicolumn{3}{c|}{\textbf{Gowalla}} & \multicolumn{3}{c}{\textbf{Yelp2018}} \bigstrut\\
    \cline{3-8}          &       & \textbf{lr} & \textbf{wd} & \textbf{others} & \textbf{lr} & \textbf{wd} & \textbf{others} \bigstrut\\
    \Xhline{1.0pt}
    \multirow{9}[2]{*}{MF} & BPR   & $10^{-3}$ & 0 &       & $10^{-3}$ & 0     &  \bigstrut[t]\\
    & LLPAUC & $10^{-1}$ & 0     & \{0.7, 0.01\} & $10^{-1}$ & 0     & \{0.7, 0.01\} \\
    & AdvInfoNCE & $10^{-1}$ & 0     & \{0.05\} & $10^{-1}$ & 0     & \{0.05\} \\
    & SL & $10^{-1}$ & 0     & \{0.025\} & $10^{-1}$ & 0     & \{0.05\} \\
    & BSL   & $10^{-1}$ & 0     & \{0.25, 0.025\} & $10^{-1}$ & 0     & \{0.1, 0.05\} \\
    & PSL-tanh & $10^{-1}$ & 0     & \{0.025\} & $10^{-1}$ & 0     & \{0.025\} \\
    & PSL-atan & $10^{-1}$ & 0     & \{0.025\} & $10^{-1}$ & 0     & \{0.025\} \\
    & PSL-relu & $10^{-1}$ & 0     & \{0.025\} & $10^{-1}$ & 0     & \{0.025\} \bigstrut[b]\\
    \Xhline{1.2pt}
    \end{tabular}
    \label{tab:appendix-ood-shift-hyperparameters}
\end{table}

\clearpage
\section{Supplementary Experiments} \label{sec:appendix-experiments-supplementary}

\subsection{Noise Results} \label{sec:appendix-experiments-results-ood-noise}

The Recall@20 and NDCG@20 results under Noise setting on four datasets (Amazon-Book, Amazon-Electronic, Amazon-Movie, Gowalla) are shown in \cref{fig:appendix-ood-noise-results-amazon-book,fig:appendix-ood-noise-results-amazon-electronic,fig:appendix-ood-noise-results-amazon-movie,fig:appendix-ood-noise-results-gowalla}.

\begin{figure}[htbp]
    \centering
    \begin{subfigure}[b]{0.48\textwidth}
        \includegraphics[width=\textwidth]{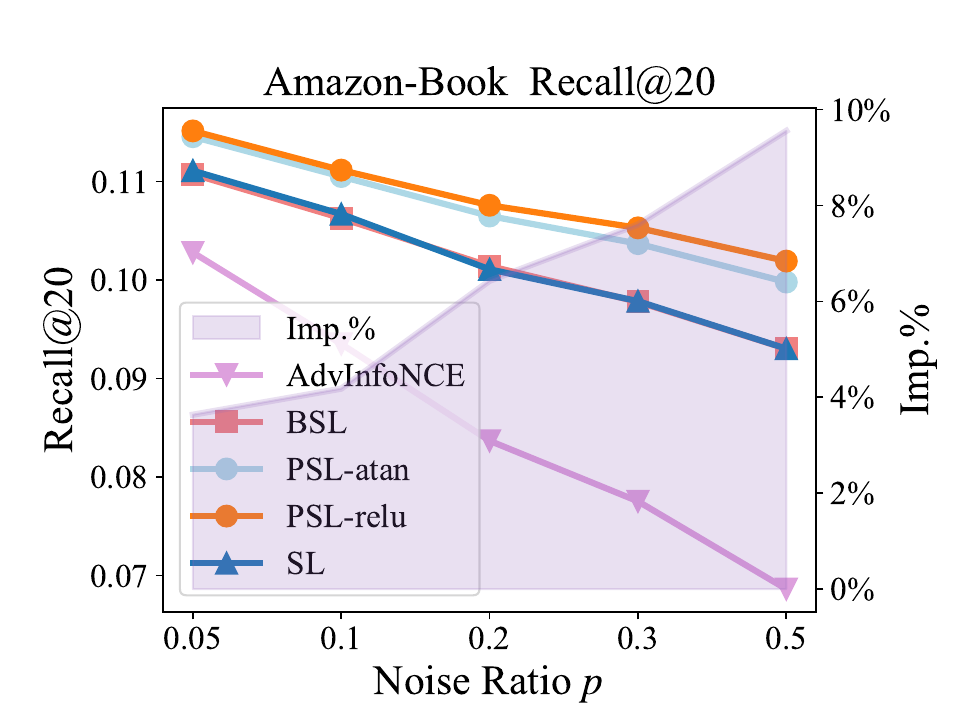}
        \caption{Amazon-Book (Recall@20)}
    \end{subfigure}
    \begin{subfigure}[b]{0.48\textwidth}
        \includegraphics[width=\textwidth]{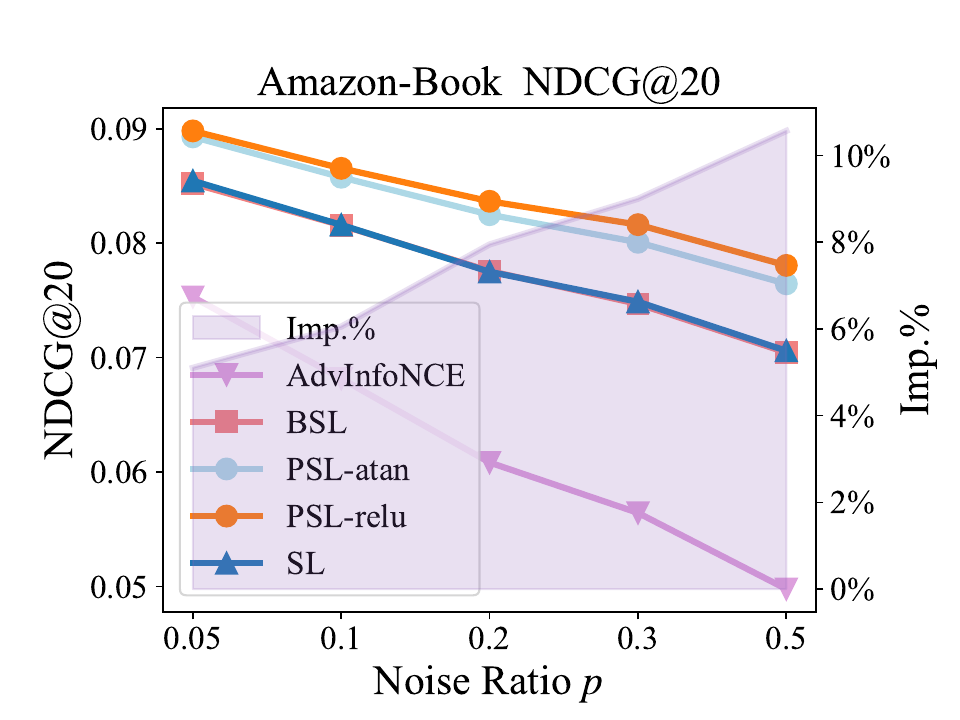}
        \caption{Amazon-Book (NDCG@20)}
    \end{subfigure}
    \caption{Noise results on Amazon-Book dataset.}
    \label{fig:appendix-ood-noise-results-amazon-book}
\end{figure}

\begin{figure}[htbp]
    \centering
    \begin{subfigure}[b]{0.48\textwidth}
        \includegraphics[width=\textwidth]{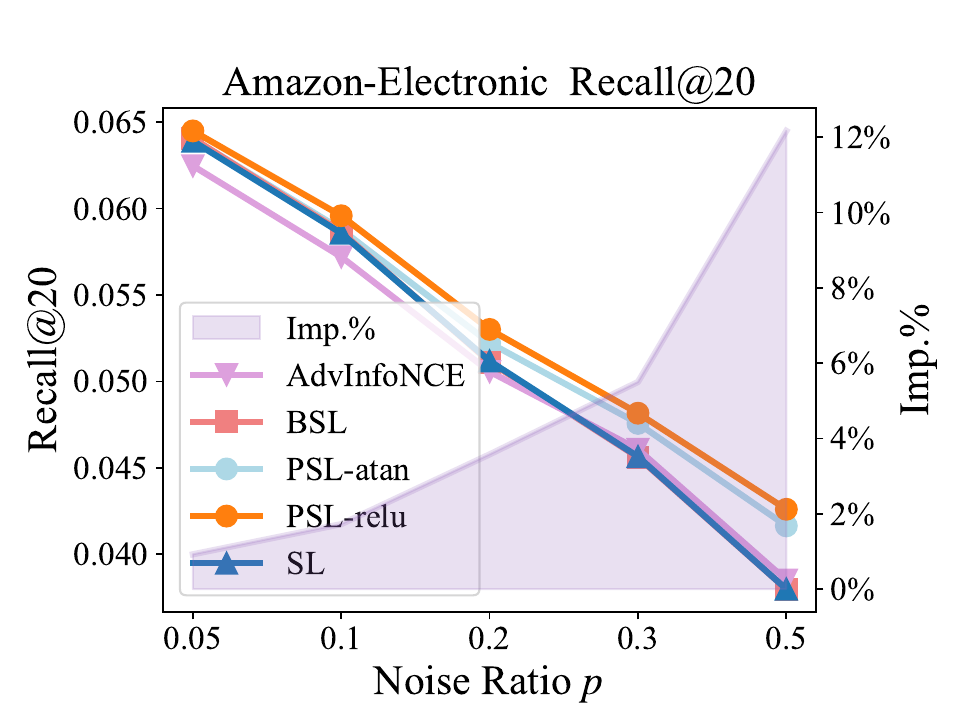}
        \caption{Amazon-Electronic (Recall@20)}
    \end{subfigure}
    \begin{subfigure}[b]{0.48\textwidth}
        \includegraphics[width=\textwidth]{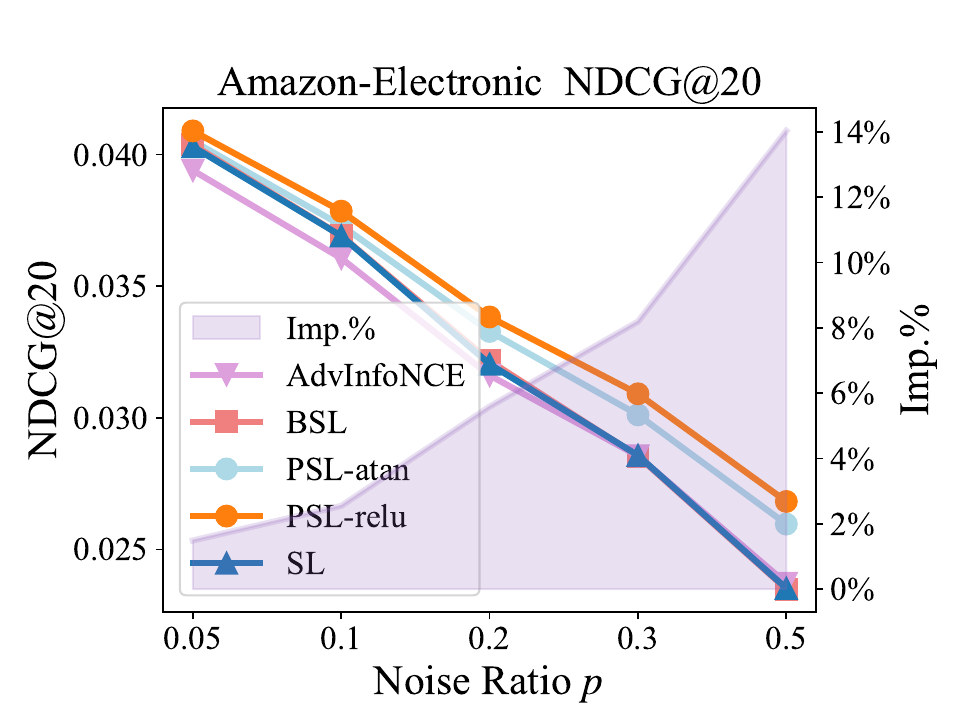}
        \caption{Amazon-Electronic (NDCG@20)}
    \end{subfigure}
    \caption{Noise results on Amazon-Electronic dataset.}
    \label{fig:appendix-ood-noise-results-amazon-electronic}
\end{figure}

\begin{figure}[htbp]
    \centering
    \begin{subfigure}[b]{0.48\textwidth}
        \includegraphics[width=\textwidth]{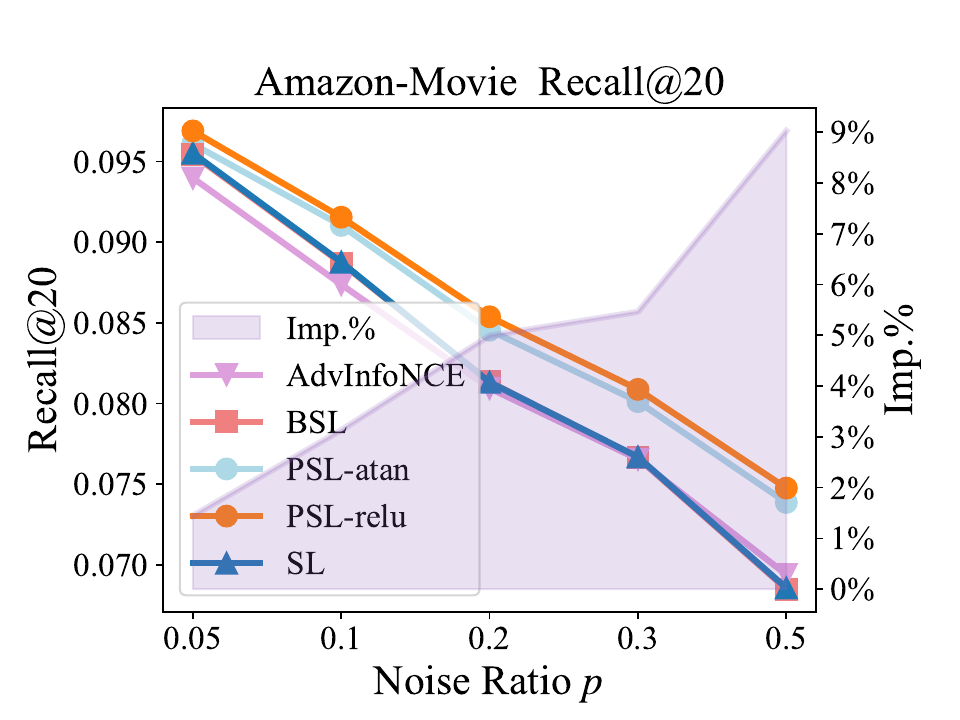}
        \caption{Amazon-Movie (Recall@20)}
    \end{subfigure}
    \begin{subfigure}[b]{0.48\textwidth}
        \includegraphics[width=\textwidth]{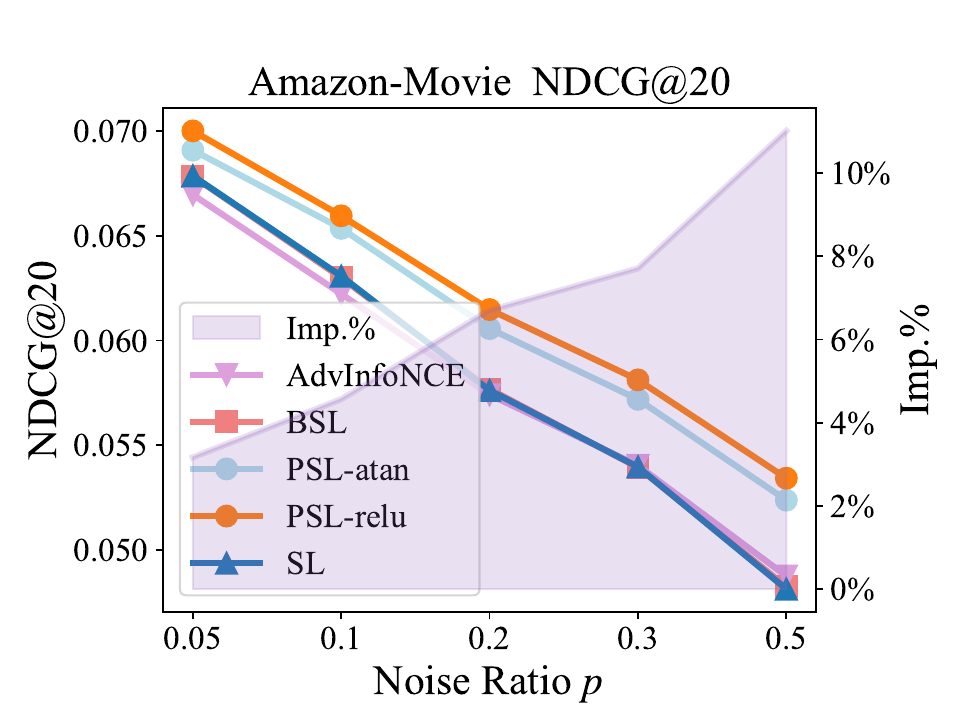}
        \caption{Amazon-Movie (NDCG@20)}
    \end{subfigure}
    \caption{Noise results on Amazon-Movie dataset.}
    \label{fig:appendix-ood-noise-results-amazon-movie}
\end{figure}

\begin{figure}[htbp]
    \centering
    \begin{subfigure}[b]{0.48\textwidth}
        \includegraphics[width=\textwidth]{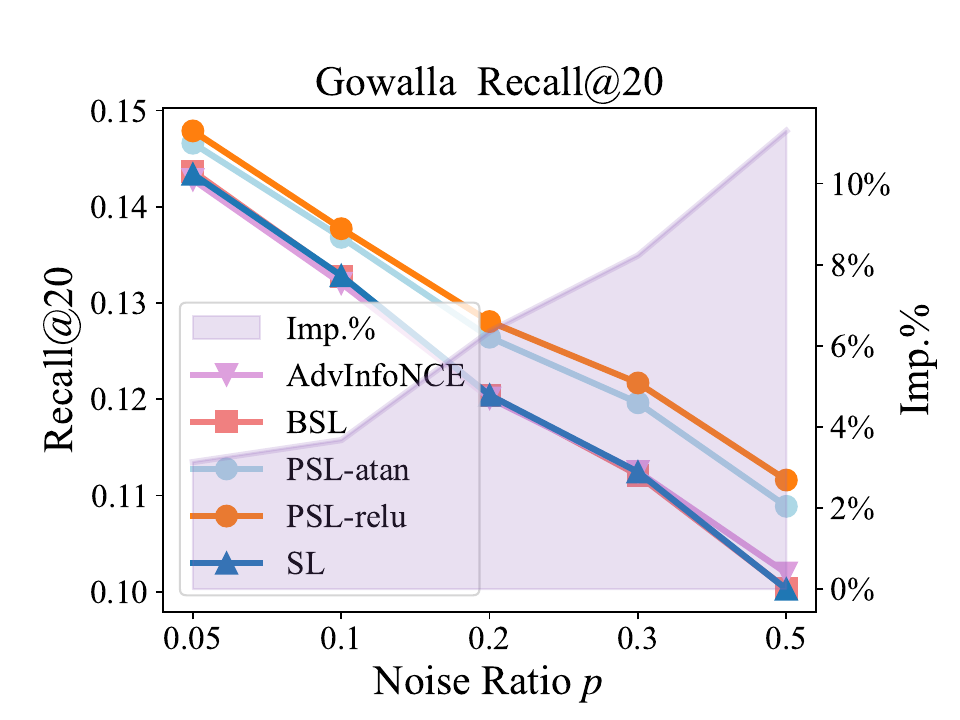}
        \caption{Gowalla (Recall@20)}
    \end{subfigure}
    \begin{subfigure}[b]{0.48\textwidth}
        \includegraphics[width=\textwidth]{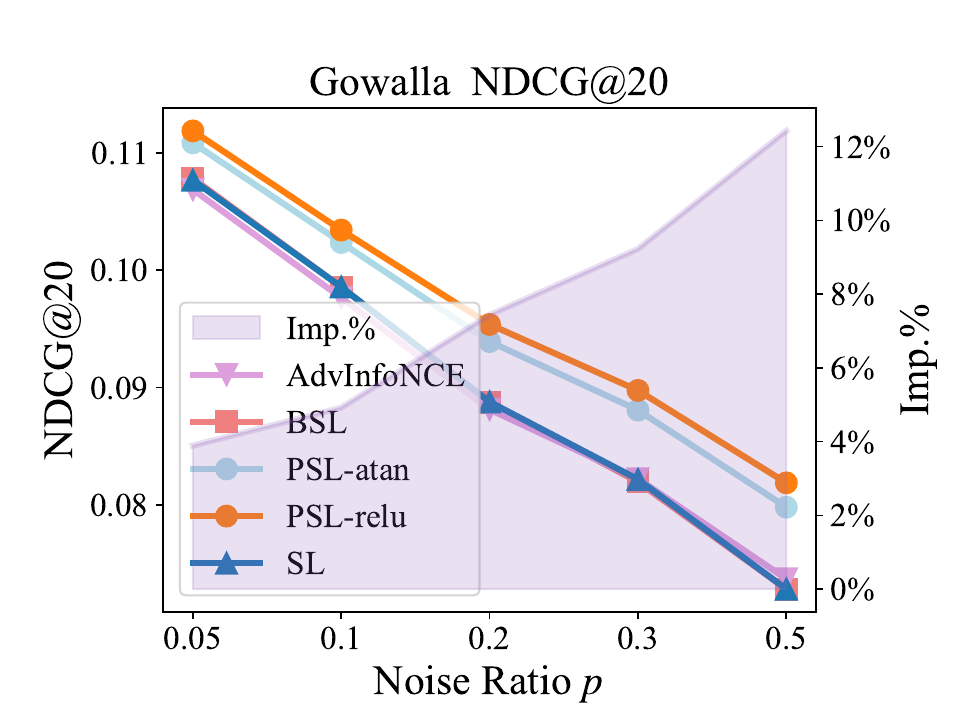}
        \caption{Gowalla (NDCG@20)}
    \end{subfigure}
    \caption{Noise results on Gowalla dataset.}
    \label{fig:appendix-ood-noise-results-gowalla}
\end{figure}

\clearpage
\subsection{PSL-softplus Results} \label{sec:appendix-experiments-results-psl-softplus}

BPR uses Softplus \citep{dugas2000incorporating} as $\log \sigma$, \ie $\sigma(d_{uij}) = \exp(d_{uij}) + 1$, which is looser than SL. That is, this surrogate activation is not a suitable choice for PSL. We call this PSL variant as \textbf{PSL-softplus}.

In this section, we conduct experiments to evaluate the performance of PSL-softplus with surrogate activation $\sigma(d_{uij}) = \exp(d_{uij}) + 1$. The IID, OOD, and Noise results of PSL-softplus are shown in \cref{tab:appendix-psl-softplus-iid-results,tab:appendix-psl-softplus-ood-shift-results}, \cref{fig:appendix-ood-noise-results-softplus-amazon-book,fig:appendix-ood-noise-results-softplus-amazon-electronic,fig:appendix-ood-noise-results-softplus-amazon-movie,fig:appendix-ood-noise-results-softplus-gowalla}, respectively. Results demonstrate that PSL-softplus is inferior to SL and three PSLs in all settings. This confirms our claim -- the choice of surrogate activation $\sigma$ is crucial, and an unreasonable or intuitive design will decrease the accuracy.

\begin{table}[htbp]
    \scriptsize
    \centering
    \caption{IID results of PSL-softplus. The results of SL, PSL-tanh, PSL-atan, and PSL-relu have been listed in \cref{tab:iid-results}. The blue-colored results are better than PSL-softplus.}
    \begin{tabular}{c|l|cc|cc|cc|cc}
    \Xhline{1.2pt}
    \multirow{2}[4]{*}{\textbf{Model}} & \multicolumn{1}{c|}{\multirow{2}[4]{*}{\textbf{Loss}}} & \multicolumn{2}{c|}{\textbf{Amazon-Book}} & \multicolumn{2}{c|}{\textbf{Amazon-Electronic}} & \multicolumn{2}{c|}{\textbf{Amazon-Movie}} & \multicolumn{2}{c}{\textbf{Gowalla}} \bigstrut\\
    \cline{3-10}          &       & \textbf{Recall} & \textbf{NDCG} & \textbf{Recall} & \textbf{NDCG} & \textbf{Recall} & \textbf{NDCG} & \textbf{Recall} & \textbf{NDCG} \bigstrut\\
    \Xhline{1.0pt}
    \multirow{5}[4]{*}{MF} & SL & \cellcolor{lightblue}0.1559  & \cellcolor{lightblue}0.1210  & 0.0821  & \cellcolor{lightblue}0.0529  & \cellcolor{lightblue}0.1286  & \cellcolor{lightblue}0.0929  & \cellcolor{lightblue}0.2064  & \cellcolor{lightblue}0.1624  \bigstrut[t]\\
    & PSL-tanh & \cellcolor{lightblue}0.1567  & \cellcolor{lightblue}0.1225  & \cellcolor{lightblue}0.0832  & \cellcolor{lightblue}0.0535  & \cellcolor{lightblue}0.1297  & \cellcolor{lightblue}0.0941  & \cellcolor{lightblue}0.2088  & \cellcolor{lightblue}0.1646  \\
    & PSL-atan & \cellcolor{lightblue}0.1567  & \cellcolor{lightblue}0.1226  & \cellcolor{lightblue}0.0832  & \cellcolor{lightblue}0.0535  & \cellcolor{lightblue}0.1296  & \cellcolor{lightblue}0.0941  & \cellcolor{lightblue}0.2087  & \cellcolor{lightblue}0.1646  \\
    & PSL-relu & \cellcolor{lightblue}0.1569  & \cellcolor{lightblue}0.1227  & \cellcolor{lightblue}0.0838  & \cellcolor{lightblue}0.0541  & \cellcolor{lightblue}0.1299  & \cellcolor{lightblue}0.0945  & \cellcolor{lightblue}0.2089  & \cellcolor{lightblue}0.1647  \bigstrut[b]\\
    \cline{2-10}          & PSL-softplus & 0.1536  & 0.1149  & 0.0826  & 0.0522  & 0.1280  & 0.0919  & 0.2053  & 0.1613  \bigstrut\\
    \Xhline{1.0pt}
    \multirow{5}[4]{*}{LightGCN} & SL & \cellcolor{lightblue}0.1567  & \cellcolor{lightblue}0.1220  & \cellcolor{lightblue}0.0823  & \cellcolor{lightblue}0.0526  & \cellcolor{lightblue}0.1304  & \cellcolor{lightblue}0.0941  & \cellcolor{lightblue}0.2068  & \cellcolor{lightblue}0.1628  \bigstrut[t]\\
    & PSL-tanh & \cellcolor{lightblue}0.1575  & \cellcolor{lightblue}0.1233  & \cellcolor{lightblue}0.0825  & \cellcolor{lightblue}0.0532  & \cellcolor{lightblue}0.1300  & \cellcolor{lightblue}0.0947  & \cellcolor{lightblue}0.2091  & \cellcolor{lightblue}0.1648  \\
    & PSL-atan & \cellcolor{lightblue}0.1575  & \cellcolor{lightblue}0.1233  & \cellcolor{lightblue}0.0825  & \cellcolor{lightblue}0.0532  & \cellcolor{lightblue}0.1300  & \cellcolor{lightblue}0.0948  & \cellcolor{lightblue}0.2091  & \cellcolor{lightblue}0.1648  \\
    & PSL-relu & \cellcolor{lightblue}0.1575  & \cellcolor{lightblue}0.1233  & \cellcolor{lightblue}0.0830  & \cellcolor{lightblue}0.0536  & \cellcolor{lightblue}0.1300  & \cellcolor{lightblue}0.0953  & \cellcolor{lightblue}0.2086  & \cellcolor{lightblue}0.1648  \bigstrut[b]\\
    \cline{2-10}          & PSL-softplus & 0.1536  & 0.1152  & 0.0814  & 0.0514  & 0.1296  & 0.0932  & 0.2053  & 0.1613  \bigstrut\\
    \Xhline{1.0pt}
    \multirow{5}[4]{*}{XSimGCL} & SL & \cellcolor{lightblue}0.1549  & \cellcolor{lightblue}0.1207  & \cellcolor{lightblue}0.0772  & \cellcolor{lightblue}0.0490  & \cellcolor{lightblue}0.1255  & \cellcolor{lightblue}0.0905  & \cellcolor{lightblue}0.2005  & \cellcolor{lightblue}0.1570  \bigstrut[t]\\
    & PSL-tanh & \cellcolor{lightblue}0.1567  & \cellcolor{lightblue}0.1225  & \cellcolor{lightblue}0.0790  & \cellcolor{lightblue}0.0501  & \cellcolor{lightblue}0.1308  & \cellcolor{lightblue}0.0926  & \cellcolor{lightblue}0.2034  & \cellcolor{lightblue}0.1591  \\
    & PSL-atan & \cellcolor{lightblue}0.1565  & \cellcolor{lightblue}0.1225  & \cellcolor{lightblue}0.0792  & \cellcolor{lightblue}0.0502  & \cellcolor{lightblue}0.1253  & \cellcolor{lightblue}0.0917  & \cellcolor{lightblue}0.2035  & \cellcolor{lightblue}0.1591  \\
    & PSL-relu & \cellcolor{lightblue}0.1571  & \cellcolor{lightblue}0.1228  & \cellcolor{lightblue}0.0801  & \cellcolor{lightblue}0.0507  & \cellcolor{lightblue}0.1313  & \cellcolor{lightblue}0.0935  & \cellcolor{lightblue}0.2037  & \cellcolor{lightblue}0.1593  \bigstrut[b]\\
    \cline{2-10}          & PSL-softplus & 0.1545  & 0.1161  & 0.0770  & 0.0484  & 0.1242  & 0.0894  & 0.1996  & 0.1557  \bigstrut\\
    \Xhline{1.2pt}
    \end{tabular}
    \label{tab:appendix-psl-softplus-iid-results}
\end{table}

\begin{table}[htbp]
    \scriptsize
    \centering
    \caption{OOD results of PSL-softplus. The results of SL, PSL-tanh, PSL-atan, and PSL-relu have been listed in \cref{tab:ood-shift-results}. The blue-colored results are better than PSL-softplus.}
    \begin{tabular}{l|cc|cc|cc|cc}
    \Xhline{1.2pt}
    \multicolumn{1}{c|}{\multirow{2}[4]{*}{\textbf{Loss}}} & \multicolumn{2}{c|}{\textbf{Amazon-CD}} & \multicolumn{2}{c|}{\textbf{Amazon-Electronic}} & \multicolumn{2}{c|}{\textbf{Gowalla}} & \multicolumn{2}{c}{\textbf{Yelp2018}} \bigstrut\\
    \cline{2-9}          & \textbf{Recall} & \textbf{NDCG} & \textbf{Recall} & \textbf{NDCG} & \textbf{Recall} & \textbf{NDCG} & \textbf{Recall} & \textbf{NDCG} \bigstrut\\
    \Xhline{1.0pt}
    SL & \cellcolor{lightblue}0.1184  & \cellcolor{lightblue}0.0815  & 0.0230  & \cellcolor{lightblue}0.0142  & \cellcolor{lightblue}0.1006  & \cellcolor{lightblue}0.0737  & \cellcolor{lightblue}0.0349  & \cellcolor{lightblue}0.0224  \bigstrut[t]\\
    PSL-tanh & \cellcolor{lightblue}0.1202  & \cellcolor{lightblue}0.0834  & \cellcolor{lightblue}0.0239  & \cellcolor{lightblue}0.0146  & \cellcolor{lightblue}0.1013  & \cellcolor{lightblue}0.0748  & \cellcolor{lightblue}0.0357  & \cellcolor{lightblue}0.0228  \\
    PSL-atan & \cellcolor{lightblue}0.1202  & \cellcolor{lightblue}0.0835  & \cellcolor{lightblue}0.0239  & \cellcolor{lightblue}0.0146  & \cellcolor{lightblue}0.1013  & \cellcolor{lightblue}0.0748  & \cellcolor{lightblue}0.0358  & \cellcolor{lightblue}0.0228  \\
    PSL-relu & \cellcolor{lightblue}0.1203  & \cellcolor{lightblue}0.0839  & \cellcolor{lightblue}0.0241  & \cellcolor{lightblue}0.0149  & \cellcolor{lightblue}0.1014  & \cellcolor{lightblue}0.0752  & \cellcolor{lightblue}0.0358  & \cellcolor{lightblue}0.0229  \bigstrut[b]\\
    \hline
    PSL-softplus & 0.1169  & 0.0799  & 0.0232  & 0.0139  & 0.0909  & 0.0665  & 0.0346  & 0.0222  \bigstrut\\
    \Xhline{1.2pt}
    \end{tabular}
    \label{tab:appendix-psl-softplus-ood-shift-results}
\end{table}

\begin{figure}[htbp]
    \centering
    \begin{subfigure}[b]{0.48\textwidth}
        \includegraphics[width=\textwidth]{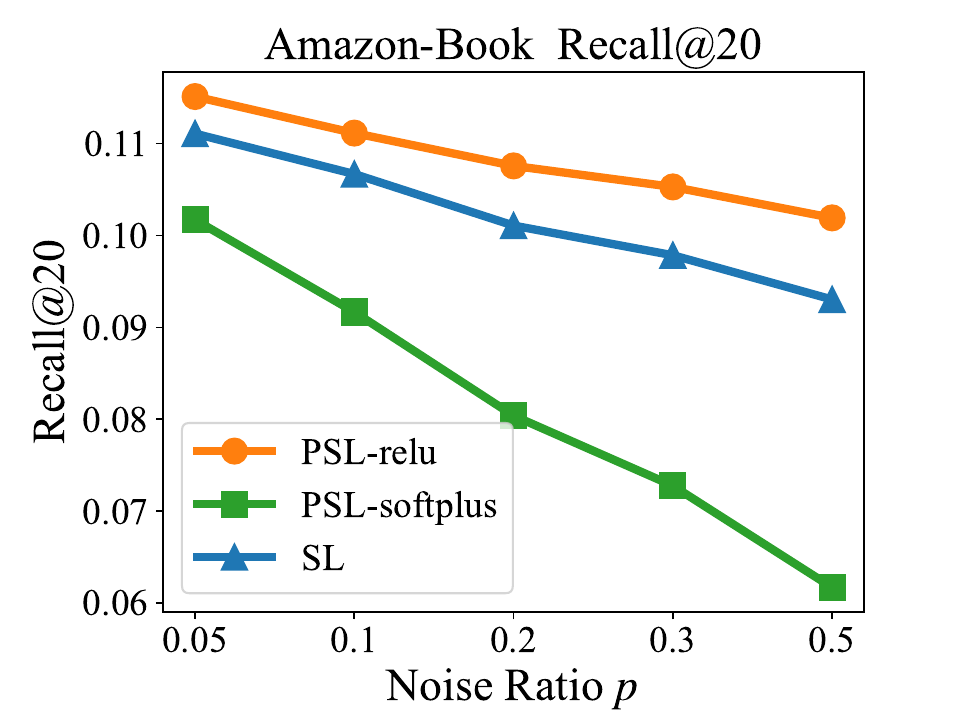}
        \caption{Amazon-Book (Recall@20)}
    \end{subfigure}
    \begin{subfigure}[b]{0.48\textwidth}
        \includegraphics[width=\textwidth]{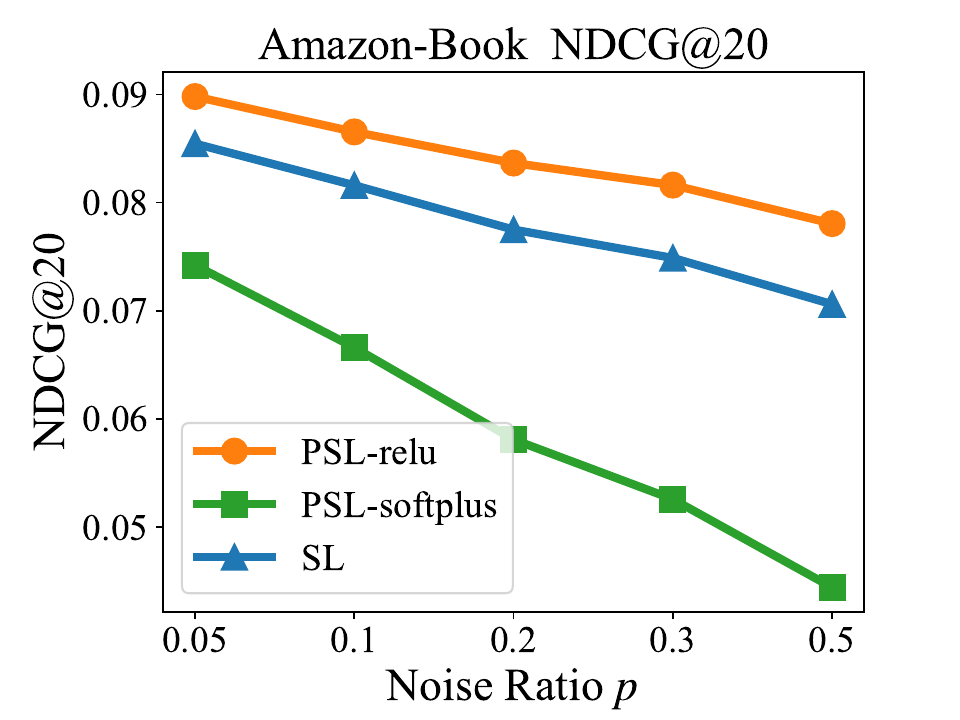}
        \caption{Amazon-Book (NDCG@20)}
    \end{subfigure}
    \caption{Noise results of PSL-softplus on Amazon-Book dataset.}
    \label{fig:appendix-ood-noise-results-softplus-amazon-book}
\end{figure}

\begin{figure}[htbp]
    \centering
    \begin{subfigure}[b]{0.48\textwidth}
        \includegraphics[width=\textwidth]{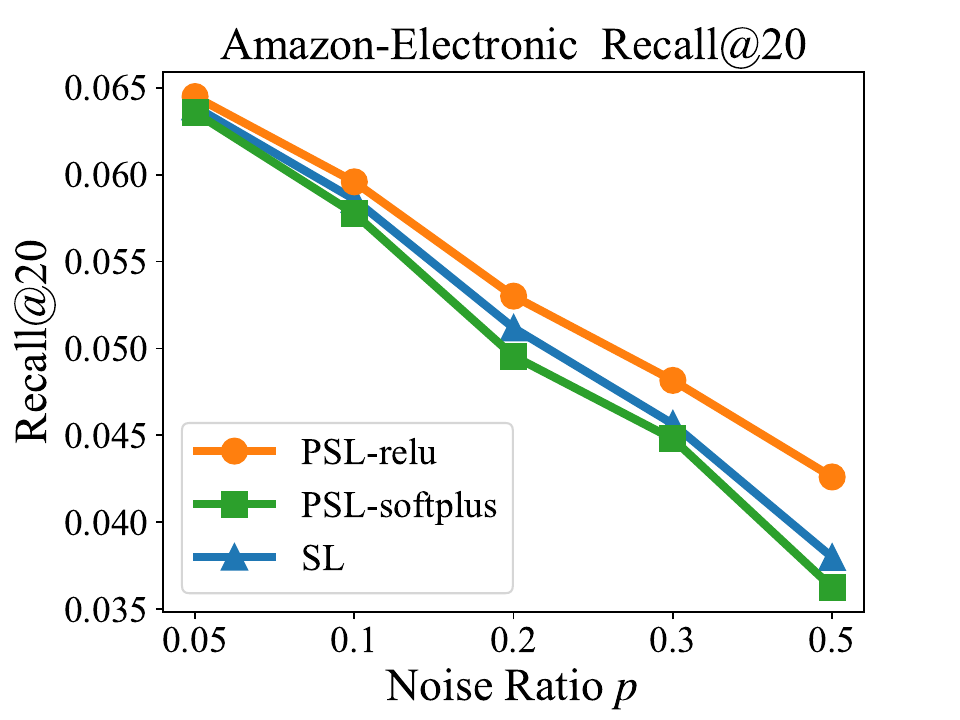}
        \caption{Amazon-Electronic (Recall@20)}
    \end{subfigure}
    \begin{subfigure}[b]{0.48\textwidth}
        \includegraphics[width=\textwidth]{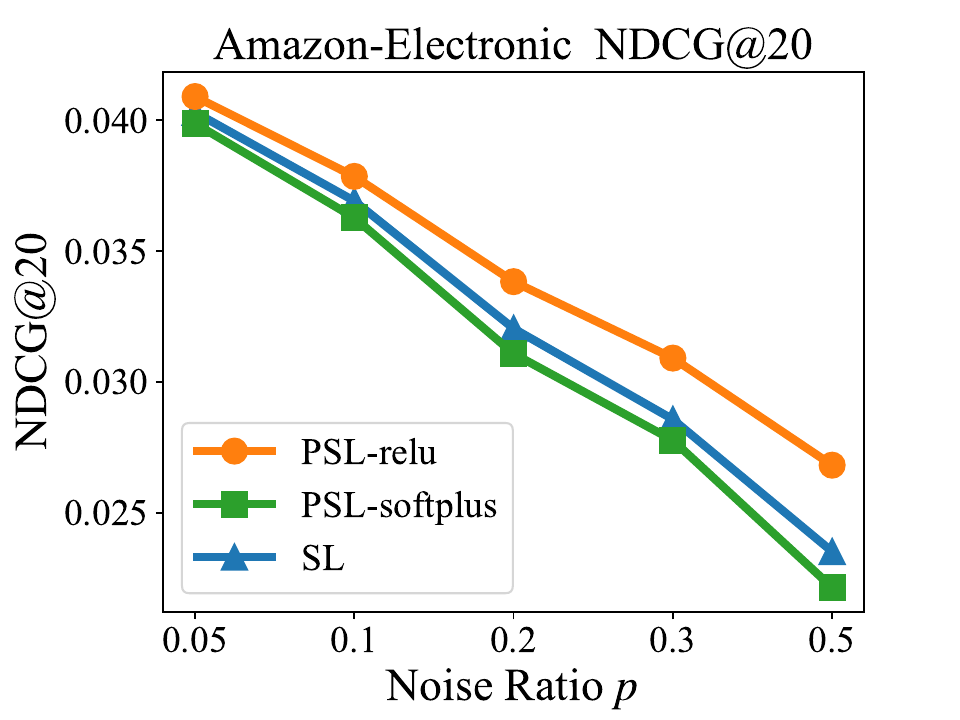}
        \caption{Amazon-Electronic (NDCG@20)}
    \end{subfigure}
    \caption{Noise results of PSL-softplus on Amazon-Electronic dataset.}
    \label{fig:appendix-ood-noise-results-softplus-amazon-electronic}
\end{figure}

\begin{figure}[htbp]
    \centering
    \begin{subfigure}[b]{0.48\textwidth}
        \includegraphics[width=\textwidth]{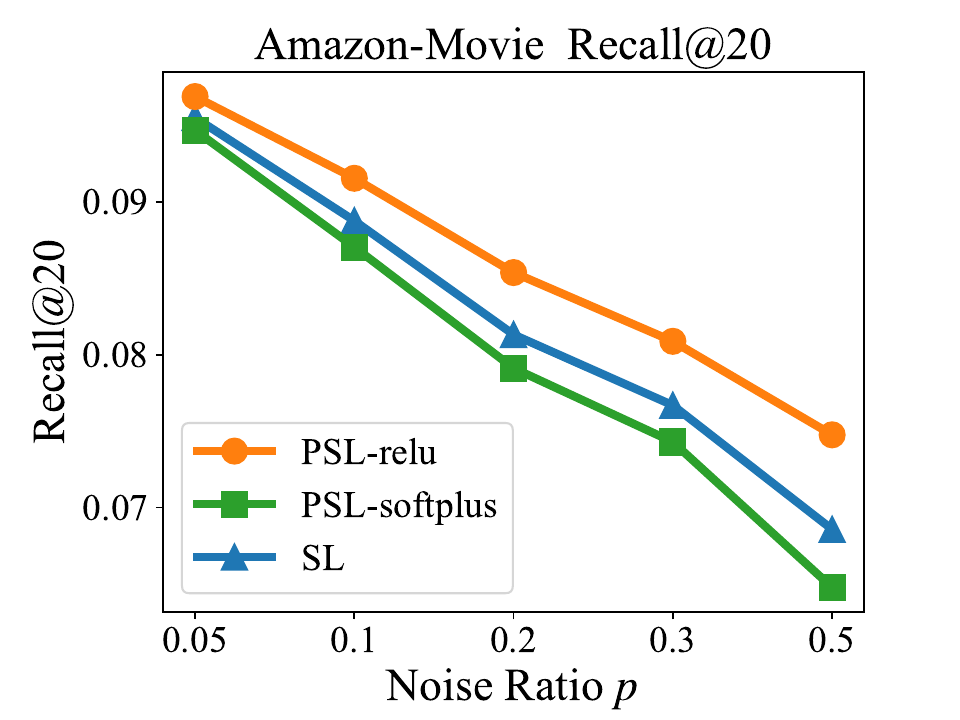}
        \caption{Amazon-Movie (Recall@20)}
    \end{subfigure}
    \begin{subfigure}[b]{0.48\textwidth}
        \includegraphics[width=\textwidth]{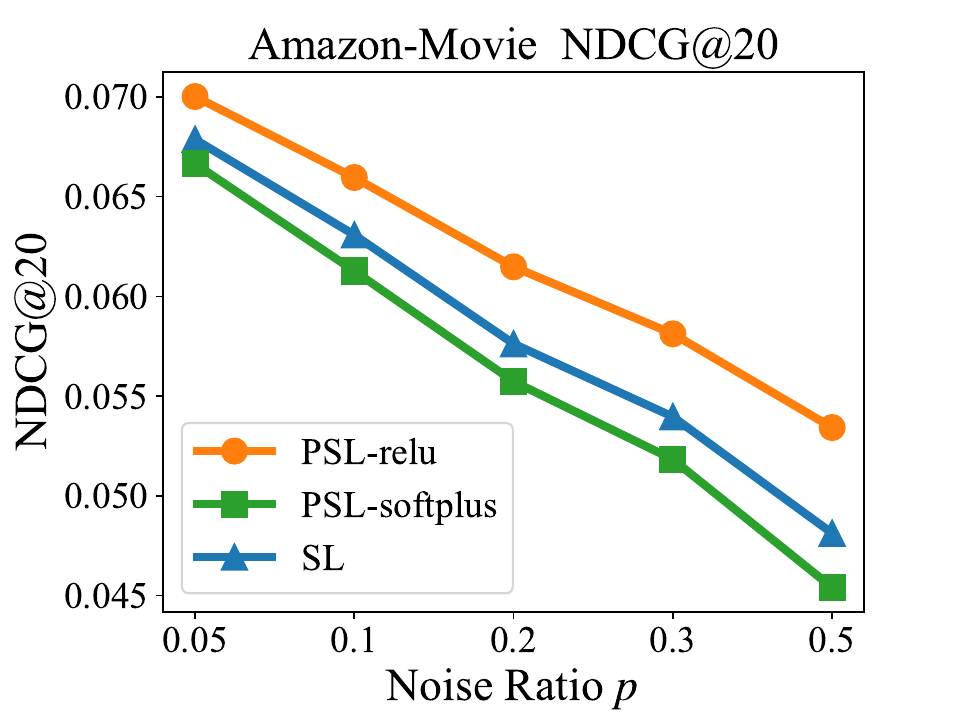}
        \caption{Amazon-Movie (NDCG@20)}
    \end{subfigure}
    \caption{Noise results of PSL-softplus on Amazon-Movie dataset.}
    \label{fig:appendix-ood-noise-results-softplus-amazon-movie}
\end{figure}

\begin{figure}[htbp]
    \centering
    \begin{subfigure}[b]{0.48\textwidth}
        \includegraphics[width=\textwidth]{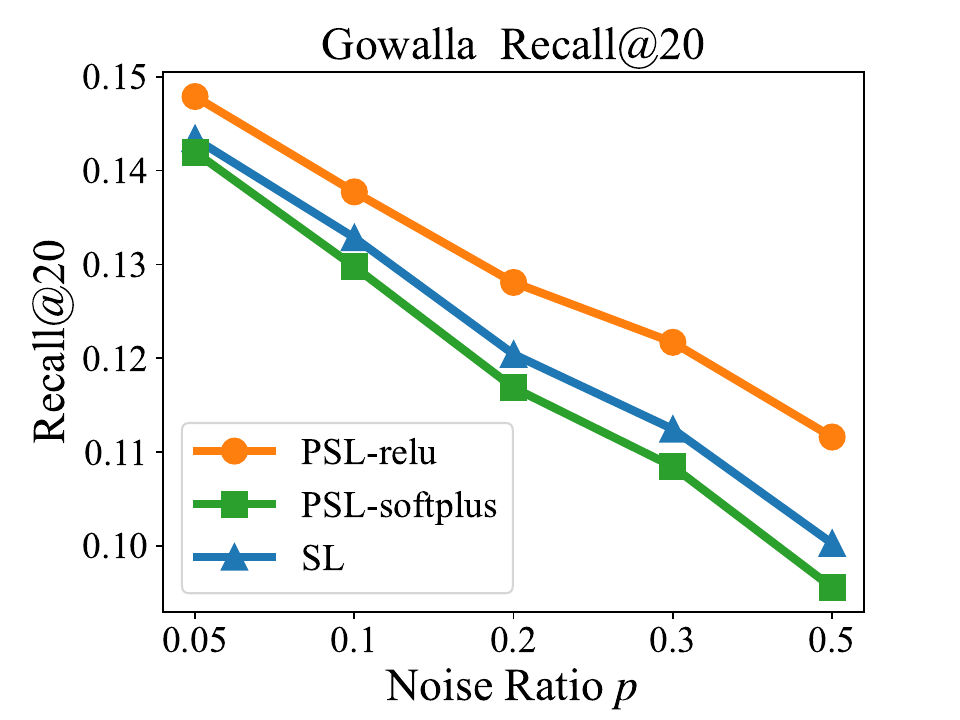}
        \caption{Gowalla (Recall@20)}
    \end{subfigure}
    \begin{subfigure}[b]{0.48\textwidth}
        \includegraphics[width=\textwidth]{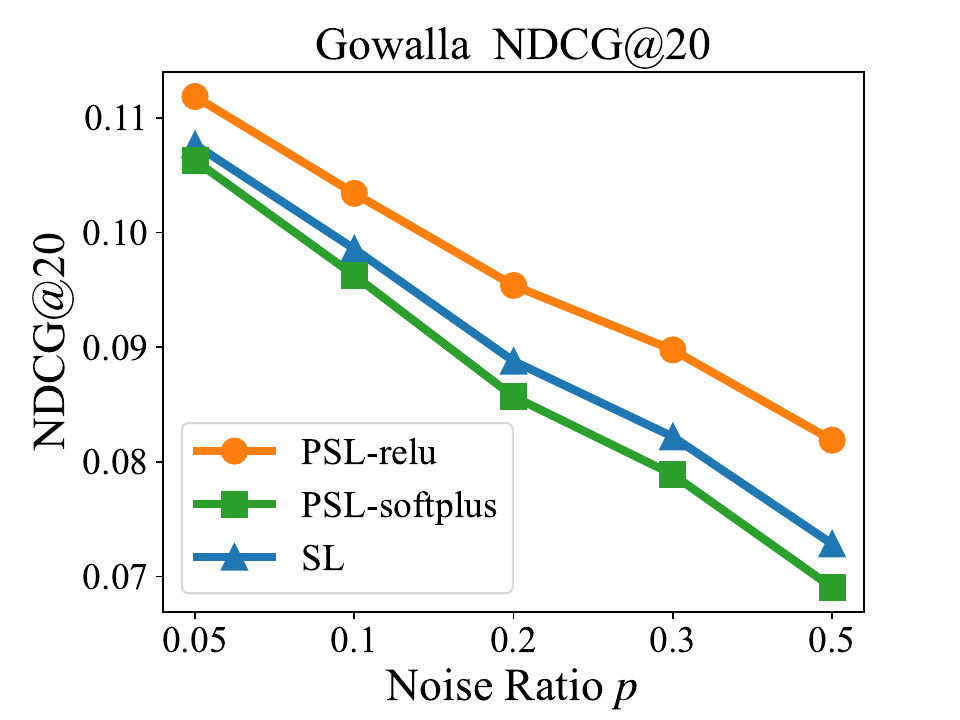}
        \caption{Gowalla (NDCG@20)}
    \end{subfigure}
    \caption{Noise results of PSL-softplus on Gowalla dataset.}
    \label{fig:appendix-ood-noise-results-softplus-gowalla}
\end{figure}

\clearpage
\subsection{Comparisons of Two Extension Forms} \label{sec:appendix-experiments-results-temperature-designs}

In this section, we compare the two different extension forms from SL to PSL, \ie \textbf{outside form} $\sigma(d_{uij})^{1/\tau}$ and \textbf{inside form} $\sigma(d_{uij} / \tau)$. As discussed in \cref{sec:psl-discussion}, the outside form scales in the value domain, while the inside form scales in the definition domain. Therefore, the inside form will lead to certain drawbacks: 1) the condition \eqref{eq:psl-activation-condition} must be satisfied over the entire $d_{uij} \in \mathbb{R}$ to ensure a tighter DCG surrogate loss (\cf \cref{lemma:psl-dcg-surrogate}), which is hard to achieve; 2) the value of $\sigma(d_{uij} / \tau)$ and its gradient may be quickly exploded when $\tau \to 0$, as the range of $d_{uij} / \tau$ is hard to control, which may cause numerical instability.

To empirically compare the above two extension forms, we conduct experiments on MF backbone and four IID datasets. Specifically, since there exists serious numerical instability, we expand the range of $\tau$ to $\{0.005, 0.025, 0.05, 0.1, 0.25, 0.5, 1.0\}$ for the inside form, where the outside form remains the same search space $\tau \in \{0.005, 0.025, 0.05, 0.1, 0.25\}$. The results are shown in \cref{tab:appendix-temperature-designs-comparison}, demonstrating that the outside form is superior to the inside form in all cases. 

\begin{table}[htbp]
    \scriptsize
    \centering
    \caption{Extension forms comparisons on MF under IID setting. The blue-colored results are better than the counterpart.}
    \begin{tabular}{c|l|cc|cc|cc|cc}
    \Xhline{1.2pt}
    \multirow{2}[4]{*}{\textbf{Form}} & \multicolumn{1}{c|}{\multirow{2}[4]{*}{\textbf{Loss}}} & \multicolumn{2}{c|}{\textbf{Amazon-Book}} & \multicolumn{2}{c|}{\textbf{Amazon-Electronic}} & \multicolumn{2}{c|}{\textbf{Amazon-Movie}} & \multicolumn{2}{c}{\textbf{Gowalla}} \bigstrut\\
    \cline{3-10}          &       & \textbf{Recall} & \textbf{NDCG} & \textbf{Recall} & \textbf{NDCG} & \textbf{Recall} & \textbf{NDCG} & \textbf{Recall} & \textbf{NDCG} \bigstrut\\
    \Xhline{1.0pt}
    \multirow{3}[2]{*}{$\sigma(d_{uij})^{1/\tau}$} & PSL-tanh & \cellcolor{lightblue}0.1567  & \cellcolor{lightblue}0.1225  & \cellcolor{lightblue}0.0832  & \cellcolor{lightblue}0.0535  & \cellcolor{lightblue}0.1297  & \cellcolor{lightblue}0.0941  & \cellcolor{lightblue}0.2088  & \cellcolor{lightblue}0.1646  \bigstrut[t]\\
    & PSL-atan & \cellcolor{lightblue}0.1567  & \cellcolor{lightblue}0.1226  & \cellcolor{lightblue}0.0832  & \cellcolor{lightblue}0.0535  & \cellcolor{lightblue}0.1296  & \cellcolor{lightblue}0.0941  & \cellcolor{lightblue}0.2087  & \cellcolor{lightblue}0.1646  \\
    & PSL-relu & \cellcolor{lightblue}0.1569  & \cellcolor{lightblue}0.1227  & \cellcolor{lightblue}0.0838  & \cellcolor{lightblue}0.0541  & \cellcolor{lightblue}0.1299  & \cellcolor{lightblue}0.0945  & \cellcolor{lightblue}0.2089  & \cellcolor{lightblue}0.1647  \bigstrut[b]\\
    \hline
    \multirow{3}[2]{*}{$\sigma(d_{uij} / \tau)$} & PSL-tanh & 0.1415  & 0.1041  & 0.0767  & 0.0494  & 0.0876  & 0.0590  & 0.1956  & 0.1507  \bigstrut[t]\\
    & PSL-atan & 0.0307  & 0.0213  & 0.0453  & 0.0268  & 0.0363  & 0.0247  & 0.0982  & 0.0727  \\
    & PSL-relu & 0.1366  & 0.1053  & 0.0723  & 0.0452  & 0.1210  & 0.0855  & 0.1732  & 0.1304  \bigstrut[b]\\
    \Xhline{1.2pt}
    \end{tabular}
    \label{tab:appendix-temperature-designs-comparison}
\end{table}


\end{document}